\documentclass{article}

\usepackage{arxiv}
\usepackage[numbers,sort&compress]{natbib}

\usepackage[utf8]{inputenc} 
\usepackage[T1]{fontenc}    
\usepackage{hyperref}       
\usepackage{booktabs}       
\usepackage{amsfonts}       
\usepackage{nicefrac}       
\usepackage{microtype}      
\usepackage[dvipsnames]{xcolor}
\usepackage{wrapfig}
\usepackage{enumitem}
\usepackage{multirow} 

\usepackage{amsmath}
\usepackage{amssymb}
\usepackage{amsthm}
\usepackage{stmaryrd}

\usepackage{natbib}
\setcitestyle{authoryear,open={(},close={)}}
\usepackage[ruled]{algorithm2e}
\makeatletter
\newcounter{framework}
\newenvironment{framework}[1][htb]
  {
   \let\c@algocf\c@framework
   \begin{algorithm}[#1]%
  }{\end{algorithm}}

\usepackage{thm-restate}

\usepackage{tikz}
\usetikzlibrary{positioning}

\DeclareMathOperator*{\argmin}{arg\,min}
\DeclareMathOperator*{\argmax}{arg\,max}
\newcommand{\R}{\mathbb{R}}
\newcommand{\N}{\mathbb N}
\newcommand{\E}{\mathbb{E}}
\newcommand{\X}{\mathcal{X}}
\newcommand{\cX}{\mathcal{X}}

\newcommand{\calH}{\mathcal{H}}
\newcommand{\cH}{\mathcal{H}}
\newcommand{\F}{\mathcal{F}}

\newcommand{\C}{\mathcal{C}}
\newcommand{\pr}{\mathbb{P}}
\newcommand{\diag}{\mathrm{diag}}
\newcommand{\diam}{\mathrm{diam}_\X}
\newcommand{\Tr}{\mathrm{Tr}}
\newcommand{\Reg}{\mathrm{Reg}}
\newcommand{\cste}{\mathfrak{C}}
\newcommand{\cM}{\mathcal{M}}
\newcommand{\cW}{\mathcal{W}}

\newcommand{\sumT}{\sum_{t=1}^T}
\newcommand{\sumK}{\sum_{k=1}^K}

\newcommand{\llVert}{\left\lVert}
\newcommand{\rrVert}{\right\rVert}
\newcommand{\sumTw}{\sum_{t\in T^w}}
\newcommand{\sumnTw}{\sum_{t\notin T^w}}

\newcommand{\ttheta}{\tilde\theta}

\newcommand{\flipi}{\mathrm{Flip}_i}

\makeatletter
\newcommand*\rel@kern[1]{\kern#1\dimexpr\macc@kerna}
\newcommand*\widebar[1]{%
  \begingroup
  \def\mathaccent##1##2{%
    \rel@kern{0.8}%
    \overline{\rel@kern{-0.8}\macc@nucleus\rel@kern{0.2}}%
    \rel@kern{-0.2}%
  }%
  \macc@depth\@ne
  \let\math@bgroup\@empty \let\math@egroup\macc@set@skewchar
  \mathsurround\z@ \frozen@everymath{\mathgroup\macc@group\relax}%
  \macc@set@skewchar\relax
  \let\mathaccentV\macc@nested@a
  \macc@nested@a\relax111{#1}%
  \endgroup
}
\makeatother

\renewcommand{\bar}{\widebar}
\renewcommand{\tilde}{\widetilde}
\renewcommand{\hat}{\widehat}
\renewcommand{\le}{\leq}
\renewcommand{\epsilon}{\varepsilon}

\usepackage[utf8]{inputenc} 
\usepackage[T1]{fontenc}    
\usepackage{hyperref}       
\usepackage{url}            
\usepackage{booktabs}       
\usepackage{amsfonts}       
\usepackage{nicefrac}       
\usepackage{microtype}      
\usepackage{xcolor}         

\title{Enjoying Non-linearity in Multinomial Logistic Bandits: A Minimax-Optimal Algorithm}

\author{%
  Pierre Boudart \\
  INRIA, École Normale Supérieure\\
  CNRS, PSL Research University\\
  Paris, France\\
  \texttt{pierre.boudart@inria.fr} \\
  \And
  Pierre Gaillard \\
  Univ. Grenoble Alpes, Inria, \\
  CNRS, Grenoble INP, LJK \\
  Grenoble, France\\
  \texttt{pierre.gaillard@inria.fr} \\
  \AND
  Alessandro Rudi \\
  SDA Bocconi School of Management\\
  Milano, Italy\\
  \texttt{alessandro.rudi@sdabocconi.it}
}

\begin{document}

\maketitle

\begin{abstract}

We consider the multinomial logistic bandit problem in which a learner interacts with an environment by selecting actions to maximize expected rewards based on probabilistic feedback from multiple possible outcomes. In the binary setting, recent work has focused on understanding the impact of the non-linearity of the logistic model (Faury et al., 2020; Abeille et al., 2021). They introduced a problem-dependent constant $\kappa_* \geq 1$ that may be exponentially large in some problem parameters and which is captured by the derivative of the sigmoid function. It encapsulates the non-linearity and improves existing regret guarantees over $T$ rounds from $\smash{O(d\sqrt{T})}$ to $\smash{O(d\sqrt{T/\kappa_*})}$, where $d$ is the dimension of the parameter space. 
We extend their analysis to the multinomial logistic bandit framework with a finite action space, making it suitable for complex applications with more than two choices, such as reinforcement learning or recommender systems. To achieve this, we extend the definition of $ \kappa_* $ to the multinomial setting and propose an efficient algorithm that leverages the problem's non-linearity. Our method yields a problem-dependent regret bound of order
$
\smash{\widetilde{\mathcal{O}}(  R d \sqrt{ {KT}/{\kappa_*}} ) }
$,
where $R$ denotes the norm of the vector of rewards and $K$ is the number of outcomes. This improves upon the best existing guarantees of order
$
\smash{\widetilde{\mathcal{O}}( RdK \sqrt{T} )}
$.
Moreover, we provide a matching $\smash{ \Omega(dR\sqrt{KT/\kappa_*})}$ lower-bound, showing that our algorithm is minimax-optimal and that our definition of $\kappa_*$ is optimal.

\end{abstract}


\section{Introduction}\label{section:introduction}

We consider the multinomial logistic (MNL) bandit problem, that unfolds as follows.
At each round $t \geq 1$, a learner chooses an action $x_t \in \X$ from an action set $\X \subset \R^d$. Then, the environment samples an outcome $y_t \in \llbracket K \rrbracket$ from the distribution $\mu(\theta_* x_t) \in \Delta_K$, where $\smash{\theta_* \in \R^{K \times d} }$ is an unknown parameter to be estimated and $\smash{ \mu: \R^K \to \Delta_K }$ the softmax function. At the end of the round, the learner receives the reward $r_t := \rho_{y_t}$, where $\rho \in \R_+^K$ is a known vector that associates a reward to each output. The goal of the learner is to minimize their expected regret defined as follows
\[
    \textstyle \Reg_T := \sum_{t=1}^T \rho^\top \big( \mu(\theta_* x_*) - \mu(\theta_* x_t) \big), \qquad \text{where} \quad x_* \in \argmax_{x \in \X} \rho^\top \mu(\theta_* x) \,.
\]
%
The MNL bandit problem falls into the umbrella of 
stochastic bandit frameworks \citep{thompson1933likelihood, robbins1952some}, which studies decision-making processes with exploration-exploitation dilemma. Linear bandits \citep{lattimore2020bandit} model a linear relationship between actions \(x_t\in\X\subseteq \R^d\) and rewards \(r_t\in\R\). They have been used with success in various applications. However they fail to model complex systems with non-linear rewards. This called for the introduction of the Generalised Linear Model (GLM) framework \citep{filippi2010parametric}. In GLMs the reward associated with an action \(x_t\in\X\) is \(\mu(\theta_* x_t) \) where \(\theta_*\) is a parameter unknown to the learner and \(\mu\) is a non linear function. The logistic bandit framework is an example of GLM obtained by choosing \(\mu\) as the sigmoid function \(\mu(z) = 1/(1+\exp(-z))\). It allows one to model situations where outcomes are evaluated by a success/failure feedback, e.g. click/no-click in ad recommendation systems. 

The MNL bandit framework \citep{amani2021ucb} is a natural extension of it. It allows to model situations with more than two outcomes. For instance consider a recommendation system on an e-commerce website. The user has several options, he may choose 1) to buy now; 2) add to the cart; 3) add to the wish-list; 4) click on "do not recommend"; 5) do not click; 6) leave the website, etc. The probability of each outcome is modelled by the softmax function \( \smash{\mu: \R^K \to [0,1]^K } \), see Section~\ref{section:problem formulation} for a formal definition. 
In this framework each outcome is associated with a specific reward $\rho_k 
\ge 0$. The goal of the learner is to give recommendations that maximise the expected reward of the outcome. Note that the MNL bandit problem is not a GLM, but a multi-index model \citep{xia2008multiple}.

\paragraph{Related work} 
A key aspect of the MNL bandit problem arises from the non-linearity of the reward.
In the binary case, where $K=2$ and $\mu$ is the sigmoid function, some works \citep{faury2020improved, jun2021improved, abeille2021instance, faury2022jointly} have focused on better understanding its impact on regret. Interestingly, this effect was shown to be captured by the constant $ \smash{\kappa := 1/\min_{\lVert\theta\rVert_2\le S} \min_{x \in \mathcal{X}} \mu'(\theta x) }$, where $S$ is an upper-bound on $\|\theta_*\|_2$,  introduced by \citet{filippi2010parametric}, who demonstrated a regret of order $\smash{\tilde O(d \kappa \sqrt{T})}$. The constant $\kappa$ can be understood as measuring the error incurred when making a linear approximation of the logistic model. Notably, $\kappa$ may be exponentially large in \(S\) and the diameter of $\mathcal{X}$, suggesting that non-linearity significantly worsens the regret guarantees compared to the linear bandit. Consequently, subsequent work has focused on improving the dependence on $\kappa$. 
\citet{faury2020improved} demonstrated that the non-linearity of the problem, i.e., $\kappa$, is not detrimental asymptotically, achieving a regret bound of order $\smash{\tilde O(d\sqrt{T} )}$. Even more strikingly, \citet{abeille2021instance} showed that one can leverage non-linearity to an advantage. They proved a regret bound scaling as $\smash{\tilde O(d\sqrt{T/\kappa_*} )}$, where $\smash{\kappa_* := 1/\mu'(\theta_* x_*)}$ measures the non-linearity at the optimum. This result represents a dramatic improvement, as in the most favorable cases, we have $\kappa_* \approx \kappa$. Moreover, they established that this bound is minimax optimal by deriving a $\smash{\Omega( d \sqrt{T/\kappa_*})}$ problem-dependent lower-bound. It is important to note that the constants $\kappa$ and $\kappa_*$ are indeed problem-dependent, as they are influenced by $S$, $\mathcal{X}$, and~$\theta_*$.

The MNL setting, which considers a reward vector $\rho \in \R^K$ with $K \geq 2$ outputs, whose norm is denoted by $ \lVert \rho \rVert_2 = R $, and where $\mu$ is the softmax function, was introduced by \citet{amani2021ucb}. They proposed a tractable algorithm that achieves a regret upper bound of order $\smash{\tilde O(RdK \sqrt{\kappa T})}$, where $\kappa$ is a generalization of the binary setting constant defined as follows\footnote{the constant $\kappa$ is originally defined slightly differently in \cite{amani2021ucb} but the definitions are equivalent up to constant factors}
\begin{equation}
    \kappa^{-1}:= \min_{\lVert\theta\rVert_2\le S} \min_{x\in\X} \lambda_{K-1}(\nabla\mu(\theta x)) \,.
    \label{eq:kappa}
\end{equation}
Interestingly, they also provided a non-tractable algorithm with a regret scaling as $ \smash{ \tilde O(RdK^{3/2}\sqrt{T}) } $. This indicates that the asymptotic dependence on $\kappa$ can also be eliminated in the MNL framework, but the question of whether this can be achieved efficiently remained open.
This question was recently addressed by \citet{zhang2024online}, who designed an efficient algorithm that achieves a regret of order $ \smash{\tilde O(RdK \sqrt{T} )} $.  An open question persists: Is it possible to extend the result of \citet{abeille2021instance} in the MNL setting and demonstrate that the non-linearity indeed yields improved asymptotic regret?

\paragraph{Main contributions} In this paper, we answer the above open question positively. 
Let $ \cX_* := \argmax_{x\in\cX} \rho^\top \mu(\theta_* x) $ be the set of optimal actions.
To quantify the non-linearity of the problem at the optimum in the multinomial setting, we generalize the problem-dependent constant $\kappa_*$ as follows:
\begin{equation}
\label{eq:kappastar}
\kappa_* = \min_{x\in\cX_*} \dfrac{\lVert \rho \rVert_2^2}{\rho^\top \nabla \mu(\theta_* x_*) \rho} \text{ when } \rho\notin \R1_K \text{ and } \kappa_* = +\infty \text{ when } \rho\in\R1_K 
\end{equation}
where $1_K \in \mathbb{R}^K$ denotes the vector of all ones, and $\mathbb{R} 1_K = \{ u 1_K \mid u \in \mathbb{R} \}$ is the one-dimensional subspace spanned by $1_K$. 
The definition of $\kappa_*$ for $\rho\in\R1_K$ is given by a continuity extension.
Note that this constant also depends on the reward vector $\rho$. As learners are expected to eventually play actions close to the optimum, $\kappa_*$ quantifies the level of non-linearity of the reward signal in the long-term regime.
The quantity $\kappa_* $ coincides with the Cramér-Rao lower-bound $\kappa_* = R^2 / \mathrm{CRLB}(\theta_* x_*)$, see Appendix~\ref{appendix:fisher information}.  It provides a rationale for why our definition of $\kappa_*$ yields minimax-optimal guarantees.
We introduce a new algorithm (Algorithm~\ref{algo:learning routine}) with a regret upper-bound in the finite action space setting given by (Theorem~\ref{thm:regret bound}):
\[
\Reg_T \le  O\left( d R \sqrt{ \dfrac{KT}{\kappa_*} }  \log(T/\delta) \right) \quad \text{w.p. } 1-2\delta 
\]
In some cases, $\kappa_*$ can be as large as $\exp(S \max_{x\in\X} \lVert x \rVert_2)$, see Appendix~\ref{appendix:large kappa star}, thereby significantly improving existing asymptotic results on MNL bandits. We prove that our regret upper-bound is minimax-optimal and that our choice of $\kappa_*$ is optimal (up to log factors) by deriving in Theorem~\ref{thm:lower bound} the following regret lower-bound $\Reg_T \ge \smash{ \Omega\big( Rd \sqrt{{KT}/{\kappa_*}} \big)} $.

We summarize existing algorithms, that focus on the dependence $\kappa$ and $\kappa_*$, for binary and MNL bandits in Table~\ref{table:regret comparison}. 

\begin{table}[!ht]
\centering
\footnotesize
\begin{tabular}{cccc}
    \hline
    Setting & Reference & Regret & Cpt. per Iter. \\
    \hline
    \multirow{5}{1.6cm}{\centering Binary} 
    & \scriptsize{\citep{filippi2010parametric}} & \( d \kappa \sqrt{T} \) & \(O(t)\) \\
    & \scriptsize{\citep{faury2020improved}} & \( d \sqrt{T} + \kappa d^2 \) & \(O(t)\) \\
    & \scriptsize{\citep{abeille2021instance}} & \( d \sqrt{T/\kappa_*} + \kappa d^2\) & \(O(t)\) \\
    & \scriptsize{\citep{faury2022jointly}} & \( d\sqrt{T/\kappa_*} + \kappa d^2 \) & \(O(\log^2 (t))\) \\
    & \scriptsize{\citep{zhang2024online}} & $ d\sqrt{T/\kappa_*} + \kappa d^2 $ &  $ O(1) $ \\
    \hline
    \multirow{7}{1.6cm}{\centering Multinomial} & \scriptsize{\citep{amani2021ucb}} & \( RdK \sqrt{\kappa T}  \) & \(O(t)\) \\
    & \scriptsize{\citep{amani2021ucb}} & \( R dK^{3/2}\sqrt{T} + \kappa d^2 K^2 \) & - \\
    & \scriptsize{\citep{lee2024improved}} & \( R d \sqrt{K \kappa T} \) & \(O(t)\) \\
    & \scriptsize{\citep{lee2024improved}} & \( Rd\sqrt{KT} + \kappa d^2 K^2 \) & - \\
    & \scriptsize{\citep{zhang2024online}} & \( RdK \sqrt{T} + \kappa d^2 K^{3/2} \) & \(O(1)\) \\
    & \scriptsize{{\bfseries (ours) - Theorem~\ref{thm:regret bound} }} & {\bfseries \( Rd  \sqrt{ {K T}/{\kappa_*}} + \kappa d^2 K^2 \)} & {\bfseries \(O(1)\)} \\
    \hline\\[5pt]
\end{tabular} 
\caption{Comparison of regret bounds for logistics and multinomial bandits, with respect to \(R, d, K, \kappa, \kappa_*\) and \(T\). For simplicity we omit logarithmic terms and other constants. For the computation cost of each algorithm we only provide the dependence in \(t\), - signifies intractable.}
\label{table:regret comparison}
\end{table}

The algorithm we introduce (Algorithm~\ref{algo:learning routine}) is computationally efficient, with a per-round complexity of order $O(1)$. A central component of our theoretical analysis involves applying the self-concordance property without incurring exponential sub-optimal factors. To this end, our algorithm first performs an exploration phase (Algorithm~\ref{algo:exploration routine}) to design a sufficiently small high-probability confidence set $\Theta$ around $\theta_*$, where the self-concordance property can be applied with only a constant factor penalty (see Section~\ref{section:exploration routine}).
Once $\Theta$ is designed, the algorithm continues to improve its estimate of $\theta_*$ by running a variant of Online Mirror Descent (OMD) constrained within $\Theta$ only. 
%
As emphasized by \citet{zhang2024online}, a central difficulty in the regret analysis lies in controlling the term 
$
\smash{\sum_{t} \rho^\top \nabla\mu(\theta_* x_t)\rho}.
$
Ideally, if $x_t \to x_*$ quickly as $t \to \infty$, this term will be of the order 
$
\smash{\sum_{t} \rho^\top \nabla\mu(\theta_* x_*)\rho = {R^2T}/{\kappa_*}},
$
leading to the final improvement in the regret. A key technical contribution of our analysis is to address this challenge by carefully leveraging the structure of the softmax function and employing the self-concordance properties within $\Theta$.

\paragraph{Multinomial \textit{Logit} Bandits}
A different line of work is the Multinomial Contextual \textit{Logit} Bandit problem \citep{cheung2017thompson, agrawal2017thompson, agrawal2019mnl, dong2020multinomial, agrawal2023tractable}, a combinatorial variant of MNL bandits that generalizes the binary logistic problem differently. At each round \(t\), the learner is asked to choose a subset of actions \( \smash{S_t \subset \llbracket K \rrbracket } \) based on observed contextual vectors \(\smash{x_{t,i} \in \mathcal{X} }\) for \(i \in \llbracket K \rrbracket\) and rewards \(\smash{\rho_{t,i} \in \mathbb{R}_+ }\). 
The goal of the learner is to maximise the expected reward modeled by the multinomial \textit{logit} model
\(
\smash{
\mathbb{E}[r_t \mid S_t] = \sum_{i \in S_t} \rho_{t,i} \exp(\theta_*^\top x_{t,i})/ \big({1 + \sum_{i \in S_t} \exp(\theta_*^\top x_{t,i})}\big), }
\)
restricted to the subset \(S_t\) of chosen actions only and where $\theta_* \in \R^d$ is a parameter unknown to the learner.
Although it may appear similar, this framework is fundamentally different: the settings differ in their parameterisation, feedback structure, and modeling assumptions.
It appears that neither framework can be easily reduced to the other. In particular, the combinatorial nature of the \textit{Logit} framework-namely, the selection of a subset $S_t$-together with the normalization in the softmax function makes any such reduction highly challenging. Moreover, in our framework, every outcome has a nonzero probability of being selected.
We provide further details in Appendix~\ref{appendix:logit bandits}.
This variant also exhibits similar challenges related to the non-linearity of the rewards and the constants \(\kappa, \kappa_*\). %
\citet{agrawal2023tractable} introduced an algorithm with \(\smash{O(d\sqrt{T})}\) regret bounds, for which the leading term is independent of \(\kappa\), representing a significant improvement over the previous bound of \(\smash{O(d\sqrt{\kappa T})} \). In the case of uniform rewards, i.e., \(\rho_{t,i} = 1\) for all \(t \in \llbracket T \rrbracket\) and all \(i \in \llbracket K \rrbracket\), \citet{perivier2022dynamic} further established a bound of \(\smash{\tilde{O}(d\sqrt{T/\kappa_*})}\).
More recently, \citet{lee2025improved} proposed an algorithm that achieves a $\mathrm{poly}(S)$-free regret of $\smash{ \tilde O (d\sqrt{T/\kappa_*}) }$ by employing adaptive exploration to exploit self-concordance.
Until now, both frameworks have been studied separately; establishing connections between them would be an interesting direction for future work.

\section{Problem Formulation}\label{section:problem formulation}
In this section, we introduce our notations and assumptions and formally recall the setting of MNL bandits.

\paragraph{Notations} Let \(1_K\in\R^K\) be the vector of 1's, $\mathbb{R} 1_K = \{ u 1_K \mid u \in \mathbb{R} \}$ is the one-dimensional subspace spanned by $1_K$, and \(\calH\) be the hyperplane supported by \(1_K\). We denote by \(\Pi : \R^K \to \R^K \) the projection on \(\calH\). We denote by $\Delta_K$ the $K$ dimensional simplex and by \(\smash{\mu: \R^K \to \Delta_K}\) the softmax function defined by
\(
\mu(z)_k \propto \exp(z_k)
\)
for all \(k\in\llbracket K \rrbracket\). 

\paragraph{Framework} The MNL bandit framework is formalised as a game of \(T\in\N\) rounds between a learner and an environment, see Framework~\ref{framework:MLB} for a short summary. 
At each round \(t\in\llbracket T \rrbracket\), the learner plays an action \(x_t\in\X\) from an action set \(\smash{\X\subseteq\R^d}\). Then, the learner observes the output of the environment \(y_t\in\llbracket K \rrbracket\) with \(K\in\N\), which are generated using the softmax function. More precisely, for all \(k\in \llbracket K \rrbracket\), we have
\( \smash{
    \pr[y_t = k | x_t] :=   \mu(\theta_* x_t)_k }
\)
where \(\smash{\theta_*\in \Pi\R^{K\times d} } \) is a parameter of the environment unknown to the learner such that \(\lVert \theta_* \rVert_2 \le S \). At the end of each round \(t\), the learner receives a reward \(\rho_{y_t}\) associated with the environment output $y_t$, from a fixed reward vector \(\smash{\rho\in\R^K_+}, \|\rho\|_2 = R\) known beforehand.
The goal of the learner is to maximise their expected reward which is equivalent to minimising the expected regret
\begin{equation*}
    \Reg_T := \sum_{t=1}^T \rho^\top \mu(\theta_*x_*) - \rho^\top\mu(\theta_*x_t)
\end{equation*}
where $ x_* := \argmax_{x\in\cX} \rho^\top \mu(\theta_* x) $ is the action maximising the expected reward.

Note that our framework differs from the original one of \citet{amani2021ucb}: instead of fixing one line of $\theta_*$ to be zero, we assume that it is such that $\smash{\sum_{k=1}^K [\theta_* x]_k = 0}$ for any $x \in \mathcal{X}$. This is ensured by the fact that $\smash{\theta_* \in \Pi \R^{K\times d} }$, which can be assumed without loss of generality since for any $\theta \in \R^{K \times d}$ and $x \in \R^d$ the probability vector of outcomes satisfies $\mu(\theta x) = \mu(\Pi \theta x)$.
Hence our model is more general, as it does not assume the existence of a dedicated no-choice (NC) item; however, such an option can be naturally incorporated by assigning a reward of 0 to any item, effectively allowing users to choose nothing. Unlike existing literature, which often makes the strong and sometimes unnecessary assumption of a universally applicable NC item, our approach removes this constraint.
While NC is appropriate in certain domains—such as e-commerce, online ads, or web search, where users frequently choose nothing—it is not suitable across the board. 
In some applications, NC isn't even feasible. For example, large language models often require explicit user preferences to proceed. Likewise, in robotics, autonomous driving, or preference-based reinforcement learning (PbRL), human feedback must indicate a choice among alternatives to guide training—NC is not an option.
In Appendix~\ref{appendix:previous framework}, we show that the framework of \citet{amani2021ucb} is included into ours.

\begin{framework}[!ht]
\caption{The Multinomial Logistic (MNL) Bandit Framework.}
\label{framework:MLB}
\For{Each time step \(t\) in \(1 \dots T\)}{
    Play action \(x_t\in\X\) \\
    Observe the decision of the environment \( y_t \in \llbracket K \rrbracket \) such that \( \pr[y_t=k|x_t] = \mu(\theta_* x_t)_k \) \\
    Get reward \(\rho_{y_t}\)
}
\end{framework}


\paragraph{Problem-dependent constants \texorpdfstring{$\kappa$}{kappa} and \texorpdfstring{$\kappa_*$}{kappa*}} As detailed in the introduction, a key aspect of the MNL bandit framework, compared to standard stochastic linear bandits, arises from the non-linearity of $\mu(\cdot)$, which appears both in the stochastic feedback model and in the reward definition.
Earlier works \cite{filippi2010parametric, amani2021ucb, zhang2024online, abeille2021instance} demonstrated that this non-linearity could be captured by two problem-dependent constants, $\kappa$ and $\kappa_*$, respectively defined in Equations~\eqref{eq:kappa} and~\eqref{eq:kappastar}, where our work introduces a new formulation of $\kappa_*$.
On the one hand, $\kappa$ quantifies the cost of performing linear approximations within the MNL framework, with larger values of $\kappa$ leading to increased regret. On the other hand, $\kappa_*$ measures the curvature at the optimum, which can be exploited in the long run to improve the asymptotic regret.
In Appendix~\ref{appendix:fisher information}, we show that our definition of $\kappa_*$ coincides with the Cramér-Rao lower-bound $\mathrm{CRLB}(\theta_* x_*)$:
\begin{equation*}
    \dfrac{R^2}{\kappa_*} := \rho^\top \nabla \mu(\theta_* x_*) \rho = \mathrm{CRLB}(\theta_* x_*) \,.
\end{equation*}
This establishes a connection between our contribution and classical optimality results in statistics.
Note that $\kappa$ is defined as the inverse of the second smallest eigenvalue of the gradient, since the smallest eigenvalue is 0 and corresponds to the eigenvector \(1_K\) composed of ones. Our definitions of $\kappa$ slightly differ from existing one due to differences in our framework notations, but they coincide with the existing definitions (see Appendix~\ref{appendix:previous framework} for details) up to constant factors. In particular, the constant $\kappa$ is shown in Appendix~\ref{appendix:bounding kappa} to be bounded from below and above as follows:
\begin{equation}\label{eq:kappa bounds}
    \dfrac{\exp(-2SX)}{K} \le \min_{\lVert\theta\rVert_2 \le S} \min_{x\in\X} \lambda_{K-1}(\nabla\mu(\theta x)) \le \dfrac{2 \exp(-2SX)}{2\exp(-SX)+(K-2)\exp(SX)} \,.
\end{equation}
Hence, $\kappa$ is exponentially large with respect to \(S \geq \|\theta_*\|\) and \(X := \max_{x \in \mathcal{X}} \lVert x \rVert_2\). 
The nonzero eigenvalues of the gradient of $\mu$ can therefore be as small as \(\kappa^{-1}\). Consequently, a naive linear approximation of the MNL framework to apply standard linear stochastic bandit analysis results in a suboptimal regret bound factor of \(\kappa\), which becomes extremely large for large values of \(X\) and \(S\). 


\paragraph{Assumptions} We use the following assumptions, which are classical in the literature \citep{amani2021ucb, zhang2024online}.
\begin{itemize}[nosep, topsep=-\parskip]
    \item The norm of each action is bounded by 1: for all \(x \in \mathcal{X}\), \(\lVert x \rVert_2 \le 1\).

    \item The reward vector  \( \rho \in \mathbb{R}_{+}^K \) satisfies \(\lVert \rho \rVert_2 = R\) and is known.

    \item The norm of the parameter \( \smash{\theta_* \in \mathbb{R}^{K \times d} } \) is bounded by \(S\): \( \lVert \theta_* \rVert_2 \le S \). The bound $S$ is known.

    \item For all \(x \in \mathcal{X}\) and for all \( \theta \) such that \(\lVert \theta \rVert_2 \le S \), we assume 
    \begin{equation}
        \label{asp:eigenvalues gradient}
        \lambda_{K-1}(\nabla\mu(\theta x)) \ge \frac{1}{\kappa} > 0 \qquad \text{and} \qquad \lambda_{1}(\nabla\mu(\theta x)) \le 1,
    \end{equation}
    where \( \lambda_{K-1} \) and \( \lambda_{1} \) denote, respectively, the second smallest and the largest eigenvalues.
\end{itemize}

Note that the assumption \(\max_{x \in \mathcal{X}} \|x\|_2 \le 1\) is made without loss of generality. Indeed, the norm of the inputs can be transferred to the norm of \(\theta_*\). 

\paragraph{Additional Notations}
Given a compact set \(\Theta\), we define its diameter under an action set \(\mathcal{X}\) as
\begin{equation*}
    \diam(\Theta) := \max_{x \in \mathcal{X}} \max_{\theta_1, \theta_2 \in \Theta} \lVert (\theta_1 - \theta_2) x \rVert_2 \,.
\end{equation*}
We denote by \(\cste\) a universal constant, i.e., a constant independent of
\(S, d, K, T, R, \kappa, \kappa_*\). The notation \(\smash{\lesssim}\) indicates an inequality up to a universal constant. 
We define the filtration
\(
\mathcal{F}_t := \{x_1, y_1, \dots, x_{t-1}, y_{t-1}, x_t\}.
\)
Throughout the paper, the index \(t\) refers to measurability with respect 
to \(\smash{\mathcal{F}_t}\), but not with respect to \(\smash{\mathcal{F}_{t-1}}\).
We denote by \(\ell_{t+1}\) the logistic loss associated with the pair \((x_t, y_t)\), defined as follows: for all \(\smash{\theta \in \R^{K\times d} }\),
\begin{equation*}
    \ell_{t+1}(\theta) := \sum_{k=1}^K - \mathbf{1}[k=y_t] \log(\mu(\theta x_t)_k) \,.
\end{equation*}

\section{Algorithm and Regret Analysis}

In this section, we introduce our algorithm (see Algorithm~\ref{algo:learning routine}) and derive a bound on its regret. The algorithm follows the explore-and-learn paradigm. Following the idea of~\citet{abeille2021instance} for binary logistic bandits, the first exploration phase aims to design a sufficiently small confidence set \(\smash{\Theta}\) around \(\theta_*\). In the second phase, the algorithm continues to improve the estimation of \(\theta_*\) while choosing the action \(x_t\) optimistically.

\subsection{Exploration Routine}\label{section:exploration routine}
We first introduce our exploration routine (see Algorithm~\ref{algo:exploration routine}) and discuss the main challenges associated with it. This exploration routine is then used as an initialisation phase in our main algorithm (see Algorithm~\ref{algo:learning routine}).

\begin{algorithm}[!ht]
\caption{\textsc{Exploration\_Routine}}
\label{algo:exploration routine}
\KwIn{Length of the procedure \(\tau\), regularisation parameter \(\lambda_0\)}
{\bfseries{Init: } \(V_0 = \lambda_0 I_{Kd} \)} \\
\For{each round \(t\) in \(1 \dots \tau\)}{
    Choose action \( \smash{x_t \in \argmax_{x\in\X} \| I_K \otimes x \|_{V_{t-1}^{-1}}} \) \\
    Observe \( y_t \sim \mu(\theta_*x_t) \)  \\
    Get reward \( \rho_{y_t} \) \\
    Update \( V_t = V_{t-1} + \tfrac{1}{\kappa} I_K \otimes x_t x_t^\top \)
}
\( \hat \theta_{\tau+1} = \argmin_{\theta\in\R^{K\times d}} \sum_{s=1}^\tau \ell_s(\theta) + \tfrac{\lambda_0}{2} \lVert \theta \rVert^2 \) \\
\bfseries{Output:} \(  \Theta := \{ \theta \in \Pi \R^{K\times d} : \lVert \theta - \hat \theta_{\tau+1} \rVert^2_{V_\tau} \le 84^2 \lambda_0 \} \)
\end{algorithm}

The goal of the exploration routine (see Algorithm~\ref{algo:exploration routine}) is to produce a confidence set \(\Theta\) such that \(\smash{\theta_* \in \Theta}\) with high probability and \(\smash{\diam(\Theta) \le 1}\). This enables us to leverage the self-concordance property \citep[Proposition~8]{sun2019generalized} of the logistic function without incurring an exponential constant. Consequently, for all \(x \in \mathcal{X}\), we have w.h.p.:
\begin{equation*}
    \nabla\mu(\theta_1 x) \le \exp(\sqrt{6} \diam( \Theta)) \nabla\mu(\theta_2 x) \leq  e \nabla\mu(\theta_2 x) \quad, \forall \theta_1, \theta_2 \in \Theta \,.
\end{equation*}
The following lemma shows that such a set \(\smash{\Theta}\) can be obtained with a reasonably small exploration length \(\tau\). The proof is deferred to Appendix~\ref{appendix:proof constant diameter}.

\begin{restatable}{lmm}{LmmConstantDiameter}\label{lemma:constant diameter}
    Let \(\delta\in(0,1]\), \( \lambda_0 = (S+1) Kd \log(T/\delta) \) and \(\smash{\tau = 336^2 \lambda_0 \kappa Kd\log\left(T\right)}\). Then, the set \(  \Theta \) returned by Algorithm~\ref{algo:exploration routine} satisfies with probability \(1-\delta\)
    \begin{equation*}
       \theta_* \in  \Theta \qquad \text{ and }\qquad  \diam( \Theta) \le 1/\sqrt{6} \,.
    \end{equation*}
\end{restatable}
As $\cX$ and $S$ are known to the learner, $\kappa$ can, in principle, be computed (see Equation~\eqref{eq:kappa}). An upper-bound can also be obtained from Equation~\eqref{eq:kappa bounds}, which is tight up to a constant factor.


\subsection{Learning Routine}
We introduce the core of our algorithm, which leverages the exploration routine (see Algorithm~\ref{algo:learning routine}). 
To select an action, we use the Optimism in the Face of Uncertainty (OFU) paradigm, a fundamental approach in bandit algorithms to address the exploration-exploitation trade-off. At each time step \(t\), the learner selects an action according to the rule 
\[
x_t \in \arg\max_{x \in \mathcal{X}} \tilde{r}_t(x),
\]
where \( \tilde{r}_t(x) \) is an optimistic reward that upper-bounds the expected reward \(\rho^\top \mu(\theta_* x)\).
In the context of logistic bandits, a common approach for defining \( \tilde{r}_t(x) \) is to construct a confidence set \(\mathcal{C}_t(\delta)\) at each round \(t\) around \(\theta_*\) and define
\begin{equation}
\tilde{r}_t(x) := \max_{\theta \in \mathcal{C}_t(\delta)} \rho^\top \mu(\theta x).
\label{eq:def_tilder_nonefficient}
\end{equation}
However, this formulation results in a non-concave maximization problem, which can be computationally challenging to solve. 
To overcome this difficulty, we adapt the optimistic reward proposed by \citet{zhang2024online} (see their Proposition~\ref{prop:optimistic reward}) who, instead of directly maximizing over the confidence set, directly express $\tilde r_t(x)$ in closed-form from an estimate of $\theta_*$ to which they add some bonus. We adapt their estimate by defining a new one \(\theta_t\) that lies within the confidence set \(\Theta\) returned by the \textsc{Exploration\_Routine} procedure (see Equation~\eqref{eq:optimistic reward}). 
Our estimate \(\theta_t\) is obtained by solving the following quadratic problem:
\begin{equation}\label{eq:probleme optim learning routine}
    \theta_{t} = \argmin_{\theta\in \Theta} \langle \nabla\ell_{t+1}(\theta_t), \theta \rangle + \tfrac{1}{2\eta} \lVert \theta - \theta_t \rVert^2_{\tilde W_t} \,,
\end{equation}
where \( \smash{ \tilde W_{t} := \sum_{s=1}^{t-1} \nabla\mu(\theta_{s+1} x_s) \otimes x_s x_s^\top + \eta \nabla \mu(\theta_t x_t) \otimes x_t x_t^\top + \lambda I_{Kd} } \), with $\eta>0$ a parameter of the algorithm. Our optimistic reward $\tilde r_t(x)$ is then obtained through a Taylor expansion of \(\mu\) and defined as follows. For all \(t\ge 1\) and \(x\in\X\), we set
\begin{equation}\label{eq:optimistic reward}
        \tilde r_t(x) := \rho^\top \mu(\theta_tx) + \epsilon_{1,t}(x) + \epsilon_{2,t}(x)\,,
\end{equation}
where 
\[
    \epsilon_{1,t}(x) := \sigma_t(\delta) \llVert \bar W_t^{-1/2} (I_K \otimes x) \nabla\mu(\theta_t x) \rho \rrVert_2 \quad \text{and} \quad \epsilon_{2,t}(x) := 3 R \sigma_t(\delta)^2 \llVert (I_K \otimes x^\top) \bar W_t^{-1/2} \rrVert_2^2  \,.
\]
Here, \( \bar{W}_t = W_t + \sum_{s=1}^t 1_K 1_K^\top \otimes x_s x_s^\top \), and \( \sigma_t(\delta) \) is a confidence term defined later in Lemma~\ref{lemma:confidence set learning}.
Closely following the proof of \citep[Proposition~1]{zhang2024online}, we show the following proposition.
\begin{restatable}{prop}{PropOtimisticReward} \label{prop:optimistic reward}
 Let $\delta \in (0,1)$. With probability \(1-\delta\), for all \(t\ge 1\) and \(x\in\X\), we have
    \begin{equation*}
        \tilde r_t(x) \ge \rho^\top \mu(\theta_* x) \quad {and} \quad  |\rho^\top \mu(\theta_* x) - \rho^\top\mu(\theta_t x) | \le \epsilon_{1,t}(x) + \epsilon_{2,t}(x)  \,.
    \end{equation*}
\end{restatable}
The key advantage of this definition of \(\tilde{r}_t(x)\) compared to the one in~\eqref{eq:def_tilder_nonefficient} is that it can be computed efficiently for any \(x\) and does not require solving any optimization problem.

We summarize our complete procedure in Algorithm~\ref{algo:learning routine} below. 

\begin{algorithm}[!ht]
\caption{\textit{REAL: Recommendation with Exploration And Learning}}
\label{algo:learning routine}
\KwIn{Exploration length \(\tau\), regularisation parameters \(\lambda_0\) and $\lambda$, step size $\eta$}
{\bfseries Init: } Run \(\Theta \gets \) \textsc{Exploration\_Routine(\(\tau, \lambda_0)\)}  \\
\hspace*{25pt} Set \(W_{\tau+1} = \lambda I_{Kd}, \bar W_{\tau+1} = \lambda I_{Kd} \) \\
\For{each round \(t\) in \(\tau+1 \dots T\)}{
    \hspace*{3pt} Choose action \( x_t \in \argmax_{x\in\X} \tilde r_t(x) \) with \(\tilde r_t(x) \) defined in Eq.~\eqref{eq:optimistic reward} \\
    Observe \( y_t \sim \mu(\theta_*x_t) \) with \(y_t \in \llbracket K \rrbracket\) \\
    Get reward \( \rho_{y_t} \) \\
    Compute $ \smash{\tilde W_t = W_t + \eta \nabla \mu(\theta_t x_t) \otimes x_t x_t^\top } $ \\
    Compute $ \smash{\theta_{t+1} = \argmin_{\theta\in\Theta} \langle \nabla\ell_{t+1}(\theta_t), \theta \rangle + \tfrac{1}{2\eta} \lVert \theta - \theta_t \rVert^2_{\tilde W_t} } $ \\
    Update \( \smash{ W_{t+1} = W_t + \nabla\mu(\theta_{t+1} x_t) \otimes x_t x_t^\top } \) \\
    Update $ \bar W_{t+1} = \bar W_{t} + \nabla\mu(\theta_{t+1} x_t) \otimes x_t x_t^\top + 1_K 1_K^\top \otimes x_t x_t^\top $
}
\end{algorithm}



\subsection{Regret analysis}\label{section:regret bound}

In this section we introduce our regret bound for Algorithm~\ref{algo:learning routine}. 
Before stating our result we need to introduce a constant $\nu >0$. It controls the variance term in the regret analysis without altering the bound’s asymptotic behavior. 
Formally $\nu$ is defined as
\begin{equation}
    \nu := \dfrac{2 \max_{x_1, x_2 \in\cX} | (\rho^{\odot 2})^\top (\mu(\theta_* x_1) - \mu(\theta_* x_2) | }{ \rho^\top \mu(\theta_* x_*) - \max_{x\in\cX \backslash \cX_*} \rho^\top \mu(\theta_* x) } \,. \label{eq:def nu}
\end{equation}
We now state our regret upper-bound. The complete proof is deferred to Appendix~\ref{appendix:proof thm regret bound}.
\begin{restatable}{thm}{ThmRegretBound} \label{thm:regret bound}
Let \(\delta \in (0,1]\) and $\cX$ be a finite action space. 
Set $\tau, \lambda_0$ as in Lemma~\ref{lemma:constant diameter}, $\eta=1$ and $\lambda=144Kd$. Then, the regret of Algorithm~\ref{algo:learning routine} satisfies, with probability at least \(1 - 2\delta\),
\begin{equation*}
    \Reg_T \le \cste R d \sqrt{\dfrac{KT}{\kappa_*}}  \log(T/\delta) + \cste \kappa (R + \nu) d^{2} K^2 \log^2(T/\delta)
\end{equation*}
where $\cste >0$ is a universal constant and $\nu >0$ is a constant defined in Equation~\eqref{eq:def nu}.
\end{restatable}
A consequence for the long-term regret is that, since the dominating term scales as \( \smash{Rd \sqrt{K T / \kappa_*}} \), the non-linearity inherent to the problem positively influences the regret bound. This contrasts with previous results from the MNL bandit literature \cite{amani2021ucb, zhang2024online, lee2024improved}, where the best known rate was \( \smash{O(RdK \sqrt{T})} \). Our approach represents a significant improvement, as in some cases \(\kappa_*\) can be exponentially large in~\(S\) (similarly to \(\kappa\)), as illustrated in the example in Appendix~\ref{appendix:large kappa star}.
It is worth pointing out that under uniform rewards, i.e., $\rho \in \mathbb{R}\mathbf{1}_K$, any algorithm incurs zero regret. In this case, the first-order term in our regret bound vanishes, since by our definition we have $\kappa_* = +\infty$.
Our result is the only one in the literature that exhibits this behavior.

The following lower-bound shows that for any number of decisions \(K\) and any dimension \(d\), there exists a problem instance where the learner incurs a regret penalty proportional to \(\smash{1/\sqrt{\kappa_*}}\). 

\begin{restatable}{thm}{ThmLowerBound}\label{thm:lower bound}
    For all \(K\ge2\), \(d\ge 2\) and any algorithm, there exist \(\theta_* \in \Pi\R^{K\times d}\) and \(\rho\in\R^K_+ \) with $\rho \notin \R1_K$ such that for \( \X= \mathcal{S}_1(\R^d) \) and for any \(T\ge d^2 \kappa_*\), the cumulative regret satisfies
    \(
        \Reg_T \ge \Omega \big(Rd \sqrt{K T/\kappa_*}\big) 
    \)
    .
\end{restatable}

Note that our probabilistic model with \(K \ge 3\) differs from the binary one, thus  our lower-bound is not a direct consequence of the binary case and requires a specific analysis, which is deferred to Appendix~\ref{appendix:proof lower bound}. 
We specifically consider a non-uniform reward, i.e. $\smash{\rho \notin \R1_K}$. For a uniform reward the regret of any algorithm is $\Reg_T = 0$ and $1/\kappa_* = 0$, which would render our lower-bound trivial.
%
Our result demonstrates that the proposed algorithm is minimax-optimal and that our choice of the non-linearity constant $\kappa_*$ is itself optimal.

\subsubsection{Confidence Set}\label{section:confidence set}
Before presenting the key ideas of the analysis of Theorem~\ref{thm:regret bound}, we first establish that the confidence levels $\sigma_t(\delta)$, which appear in the definitions of the bonuses added to the reward (see Equation~\eqref{eq:optimistic reward}), are sufficiently small.  These levels are intrinsically linked to the size of the confidence set constructed around $\theta_*$ at each round. For each time step \(t\ge \tau+1\), the pair \( \smash{(\theta_{t+1}, \bar W_{t+1})} \) is associated with the confidence set
\begin{equation*}
    \C_t(\delta) := \left\{ \theta : \lVert \theta - \theta_{t+1} \rVert_{\bar W_{t+1}} \le \sigma_t(\delta) \right\}
\end{equation*}
where \( \bar W_{t+1} = W_{t+1} + \sum_{s=1}^t 1_K 1_K^\top \otimes x_s x_s^\top \).
Leveraging the fixed diameter set we build in the exploration phase and using \citep[Theorem~4.2]{lee2025improved}, we provide a $\mathrm{poly}(S)$-free confidence set.
In the following lemma, we show that \(\theta_* \in \C_t(\delta) \) with high probability. The proof is deferred to Appendix~\ref{appendix:proof confidence set learning}.

\begin{restatable}{lmm}{LmmConfidenceSetLearning}\label{lemma:confidence set learning}
    Let $\delta \in (0,1]$. Set $\eta=1$ and $ \lambda = 144Kd $. Let us assume Lemma~\ref{lemma:constant diameter} holds. Let us define $\sigma_t(\delta) = \tfrac{2}{\sqrt{6}} \sqrt{Kd \log(t/\delta)} + 2 S \sqrt{\lambda} $. Then we have with probability $1-\delta$, for all $t\ge 1$, 
\begin{equation*}
    \lVert \theta_* - \theta_{t+1} \rVert_{\bar W_{t+1}} \le \sigma_{t}(\delta) \,.
\end{equation*}
\end{restatable}



\subsubsection{Proof Sketch of Theorem~\ref{thm:regret bound}}
We start by using a classical OFU argument. Using Proposition~\ref{prop:optimistic reward} together with the definition of $\smash{x_t \in \argmax_x \tilde r_t(x) }$,  we bound the regret as
\begin{equation}
    \Reg_T \leq  \tau + \sum_{t=\tau+1}^T  \rho^\top (\mu(\theta_*x_*) - \mu(\theta_* x_t)) \le \tau  + 2 \sum_{t=1}^T \epsilon_{1,t}(x_t)  + 2 \sum_{t=1}^T \epsilon_{2,t}(x_t) 
    \label{eq:sketch_proof_thm4_1}
\end{equation}
where $\epsilon_{1,t}$ and $\epsilon_{2,t}$ are the bonuses defined below Equation~\eqref{eq:optimistic reward}.
The first term  $\tau$ corresponds to the exploration cost and yields the logarithmic term in $T$ in the regret upper-bound. The sum $\sum_{t} \epsilon_{2,t}$ is bounded with standard linear algebra. Defining $\smash{U_t := \frac{1}{\kappa} \sum_{s=1}^t I_K \otimes x_s x_s^\top + \frac{\lambda}{2} I_{Kd}}$, we have $U_t \preccurlyeq \bar W_t$ (which justifies the choice of $\bar W_t$ instead of $W_t$ in the analysis), which entails
\begin{align}
    \sum_{t=1}^T \epsilon_{2,t}(x_t) & =  3 R \sum_{t=1}^T \sigma_t(\delta) \big \lVert (I_K \otimes x_t^\top) \bar W_t^{-1/2} \big\rVert_2^2  \lesssim  R \kappa  \sigma_T(\delta) \sum_{t=1}^T \Tr\left( \left( \tfrac{1}{\kappa} I_K \otimes x_t x_t^\top\right) \bar W_t^{-1} \right) \nonumber \\
   &  \leq R \kappa \sigma_T(\delta) \sum_{t=1}^T \Tr((U_t - U_{t-1})U_t^{-1}) \leq   R \kappa \sigma_T(\delta) \sum_{t=1}^T \log \frac{|U_t|}{|U_{t-1}|} \nonumber  \\
   & \lesssim R \kappa K^2 d^2 \log^2(T/\delta) \,. \label{eq:sketch_proof_Thm4_2}
\end{align}
Controlling the other sum $\sum_t \epsilon_{1,t}$ is more challenging. Careful derivations followed by Cauchy-Schwarz inequality lead to
\begin{equation}
    \sum_{t=1}^T \epsilon_{1,t}(x_t) \lesssim \sqrt{\sigma_T(\delta)} \sqrt{\sum_{t=1}^T \lVert \bar W_t^{-1/2} (I_K \otimes x_t) \nabla\mu(\theta'_t x_t)^{1/2} \rVert_2^2} \sqrt{ \sum_{t=1}^T \rho^\top \nabla\mu(\theta_* x_t) \rho} \,.
    \label{eq:sketch_proof_Thm4_3}
\end{equation}
The first sum in the square root may again be controlled in $O(d\log T)$, i.e. $K$-free, through a careful linear algebra analysis of the eigenvalues and a Trace-Determinant argument. The second sum is a standard term that appears in earlier work. Indeed, a key step in achieving minimax optimal rates in the binary setting \citep{abeille2021instance, faury2022jointly} involves proving that
\begin{equation*}
     \sum_{t=1}^T \mu'(\theta_*^\top x_t)  \le T / \kappa_* + \Reg_T \,.
\end{equation*}
In the MNL setting, \citet[Appendix~C.5]{zhang2024online} also showed that
\begin{equation}
     \sum_{t=1}^T \rho^\top \nabla\mu(\theta_* x_t) \rho  \le R^2 T / \kappa_* + 2 \Reg_T   \,,
    \label{eq:sufficient_eq}
\end{equation}
was sufficient to obtain a regret with a $1/\kappa_*$ dependence. 
However, as they admit, such a relationship is unclear in general and challenging to establish. Indeed, in the binary setting, the analysis by \citet{abeille2021instance} heavily relies on specific properties of the one-dimensional sigmoid function $\mu$, which  satisfies \( | \mu'' | \le \mu' \). These properties do not carry over to the multi-dimensional setting when $\mu$ is the softmax function.
Moreover, in the binary setting, since the sigmoid function is increasing, the optimal decision $\smash{x_* \in \argmax_{x \in \mathcal{X}} \{ \mu(\theta_*^\top x) \} }$ can be easily expressed as the solution to the linear optimization problem $\smash{\argmax_{x \in \mathcal{X}} \{ \theta_*^\top x \} }$. This no longer holds because $\mu$ is multi-dimensional and because $x_*$ also depends on the reward vector $\rho$. Due to this difficulty, instead of~\eqref{eq:sufficient_eq}, \citet{zhang2024online} show that
\begin{equation*}
     \sum_{t=1}^T \rho^\top \nabla\mu(\theta_* x_t) \rho  \le R^2 T / \kappa_* + 2 R \Reg_T  +  \sum_{t=1}^T \sum_{k=1}^K \rho_k^2 ( \mu(\theta_* x_t)_k - \mu(\theta_* x_*)_k) \,.
\end{equation*}
The difficulty, as pointed out in \cite{zhang2024online}, is that the last term may be non-negative and significantly higher than the regret. To circumvent this problem, we derive a slightly different upper-bound that replaces $\Reg_T$ in Equation~\eqref{eq:sufficient_eq} with an upper-bound obtained from the reward bonuses \(\ \epsilon_{1,t}(x_t)\) and \(\epsilon_{2,t}(x_t)\).
We add and subtract $\rho^\top \nabla\mu(\theta_* x_*) \rho$. The introduction of the constant $\nu$ allows us to handle this quantity while preserving the same asymptotic properties.
Carefully controlling the difference term we establish:
\begin{align*}
    \sum_{t=1}^T \rho^\top \nabla\mu(\theta_* x_t) \rho 
    & =  \sum_{t=1}^T \langle \rho, \nabla\mu(\theta_*x_*)\rho\rangle +  \left\langle\rho, (\nabla\mu(\theta_* x_t)-\nabla\mu(\theta_* x_*)) \rho \right\rangle \\
    & \leq \dfrac{R^2 T}{\kappa_*} + (4R + \nu) \sum_{t=1}^T \left(\epsilon_{1,t}(x_t) + \epsilon_{2,t}(x_t) \right) \,.
\end{align*}
The proof concludes by combining this with Equations~\eqref{eq:sketch_proof_Thm4_2} and~\eqref{eq:sketch_proof_Thm4_3}, solving a second-order equation of the form
\[
    \sum_{t=1}^T \left(\epsilon_{1,t} + \epsilon_{2,t}\right) \leq  C_1 + C_2 \sqrt{ \dfrac{R^2 T}{\kappa_*} + 1 + (4R + \nu) \sum_{t=1}^T \left(\epsilon_{1,t} + \epsilon_{2,t}\right) },
\]
and substituting the solution into the initial regret bound~\eqref{eq:sketch_proof_thm4_1}.

\subsection{Adaptive exploration and changing action sets}

The initial exploration phase of our algorithm might be concerning from a practical viewpoint. It enforces \(\kappa\) rounds of exploration which given the nature of \(\kappa\) might be costly. 
In Appendix~\ref{appendix:proof adaptive regret bound}, we present a variant of our algorithm (see Algorithm~\ref{algo:adaptive exploration}) that employs adaptive rather than hardcoded exploration based on \cite{lee2025improved} work. This adaptive approach enables the extension of our framework to non-stationary action sets $\cX_t \subseteq \cX$. 
We adapt our definition of the non-linearity constant $\kappa_*$ to match the action sets $\cX_t$:
\begin{equation*}
    \kappa_{*,t} = \min_{x_{*,t}\in\cX_{*,t}} \dfrac{\lVert \rho \rVert_2^2}{\rho^\top \nabla \mu(\theta_* x_{*,t}) \rho} \text{ when } \rho\notin \R1_K \text{ and } \kappa_{*,t} = +\infty \text{ when } \rho\in\R1_K 
\end{equation*}
where $ \cX_{*,t} := \argmax_{x\in\cX_t} \rho^\top \mu(\theta_* x) $ is the set of optimal actions at time $t$.
We also modify the regret definition to take $\cX_t$ into account:
\begin{equation*}
    \Reg_T := \sumT \rho^\top \mu(\theta_* x_{*,t}) - \rho^\top \mu(\theta_* x_t) 
\end{equation*}
where $x_{*,t}\in \cX_{*,t}$.
The algorithm is based on a trigger condition. 
Let $T^w \subseteq [T]$ denote the set of exploration steps of the algorithm. 
At any time step $t$, the algorithm performs an exploration step if the following condition is satisfied:
\[
\max_{x\in\cX_t} \lVert I_K \otimes x \rVert^2_{(H_{t-1}^w)^{-1}} \ge \frac{1}{\tau_t^2}  \, \qquad \text{where}
\quad  H_{t-1}^w = \sum_{s=1}^{t-1} \tfrac{1}{\kappa} I_K \otimes x_s x_s^\top \mathbf{1}\{ s\in T^w\} \,.
\]
Each time the algorithm explores, it refines its estimate of $\theta_*$ and updates the corresponding confidence set. 
Otherwise, it follows the learning procedure described in Algorithm~\ref{algo:learning routine}.

We now introduce our regret bound for Algorithm~\ref{algo:adaptive exploration}. The proof is deferred to Appendix~\ref{appendix:proof adaptive regret bound}.

\begin{restatable}{thm}{ThmAdaptiveRegretBound}\label{thm:adaptive regret bound}
    Let $\delta \in (0,1]$ and $ (\cX_t)_{t=1}^T $ be finite action sets. Set $\lambda^w = 72(1+\sqrt{6}S)Kd, \eta^w = (1+\sqrt{6}S)/2$ and $\lambda=144Kd$.
    Then, the regret of Algorithm~\ref{algo:adaptive exploration} satisfies with probability at least $1-2\delta$,
    \begin{equation*}
        \Reg_T \le \tilde O\left( Rd \sqrt{ K \sumnTw \dfrac{1}{\kappa_{\ast, t}} }  \right) 
    \end{equation*}
    where $T^w$ is the set of time steps when the algorithm explores.
\end{restatable}
In the case of stationary arm-sets $\smash{\cX_t= \cX}$, we recover the regret guarantee of Theorem~\ref{thm:regret bound}, obtaining a regret upper-bound of $\smash{\tilde O(Rd\sqrt{KT/\kappa_*})}$. In the non-stationary case, we obtain $\smash{\sqrt{T} \sqrt{\tfrac{1}{T}\sumnTw \tfrac{1}{\kappa_{*,t}} }  }$, replacing the non-linearity constant in the optimum by its on-trajectory average version.

\section{Conclusion}

This work establishes that non-linearity in multinomial logistic bandits can be leveraged to improve asymptotic regret guarantees, extending results previously known only for the binary setting. We introduce a new problem-dependent constant $\kappa_*$ and design an algorithm that achieves minimax-optimal regret bounds of order $\smash{\tilde{O}( Rd \sqrt{ {KT}/{\kappa_*}}  )}$ in the finite action space setting, while preserving computational efficiency. Crucially, we also prove a matching lower-bound of $\smash{\Omega( Rd \sqrt{ {KT}/{\kappa_*}} )}$, thereby demonstrating that both our algorithm and our definition of $\kappa_*$ is optimal up to logarithmic factors. Our analysis relies on a tailored exploration strategy and exploits the self-concordance property of the softmax function, enabling tighter control of curvature effects at the optimum. These findings demonstrate that non-linearity, rather than being a limitation, can serve as a structural advantage in sequential decision-making.

\paragraph{Acknowledgements.}
We thank Francis Bach for his precious knowledge and the anonymous reviewers for the insightful comments. 
We are grateful to Linzhe He for pointing out an error in an earlier version of this work.
A.R. acknowledges the support of the French government under management of Agence Nationale de la Recherche as part of the “Investissements d’avenir” program, reference ANR-19-P3IA-0001 (PRAIRIE 3IA Institute) and the support of the European Research Council (grant REAL 947908).
This work was supported by funding from the French government, managed by the National Research Agency (ANR), under the France 2030 program, reference ANR-23-IACL-0006.



\bibliography{bibfile}

\newpage
\appendix
\begin{center}
    \huge
    APPENDIX
\end{center}

This appendix is organised as follows:
\begin{itemize}[nosep, leftmargin=*]
    \item[-] Appendix~\ref{appendix:notations}: Notations
    \item[-] Appendix~\ref{appendix:bounds on constants kappa}: Bounds on the Constants $\kappa$ and $\kappa_*$
    \item[-] Appendix~\ref{appendix:fisher information}: Linking $\kappa_*$ with the Fisher Information
    \item[-] Appendix~\ref{appendix:previous framework}: Comparison with the Framework of \citet{amani2021ucb}
    \item[-] Appendix~\ref{appendix:logit bandits}: Discussion of the Multinomial \textit{Logit} Bandits
    \item[-] Appendix~\ref{appendix:analysis algorithm}: Analysis of Algorithm~\ref{algo:learning routine}
    \item[-] Appendix~\ref{appendix:proof lower bound}: Proof of Theorem~\ref{thm:lower bound} - Lower bound
    \item[-] Appendix~\ref{appendix:removing the exploration}: Removing the Exploration
    \item[-] Appendix~\ref{appendix:auxiliary results}: Auxiliary Results
\end{itemize}

\section{Notations}\label{appendix:notations}
We detail below useful notations and basic properties used throughout the appendix.
\begin{itemize}
    \setlength\itemsep{0.3cm}
    \item[-] \(\llbracket T \rrbracket := \{1, 2, \dots, T \} \quad, \forall T \in \N^* \)
    \item[-] \(\cste : \text{Universal constant, i.e. independent of \(S,d,K, T, \kappa, \kappa_*\)} \)
    \item[-] \(\kappa_*^{-1} = \tfrac{\rho^\top \nabla \mu(\theta_* x_*) \rho}{\lVert \rho \lVert_2^2} \)
    \item[-] \(\kappa := \max_{\lVert\theta\rVert\le S} \max_{x\in\X}  \tfrac{1}{\lambda_{K-1}(\nabla\mu(\theta x))} \)
    \item[-] \(\ell_{t+1}(\theta) := \sum_{k=1}^K - \mathbf{1}[k=y_t] \log(\mu(\theta x_t)_k) \)
    \item[-] \(\diam(\Theta) = \max_{x\in\X}\max_{\theta_1, \theta_2 \in \Theta} \lVert (\theta_1 - \theta_2) x \rVert_2 \)
    \item[-] \(H_t(\theta) := \sum_{s=1}^t \nabla\mu(\theta x_s) \otimes x_s x_s^\top + \lambda_0 I_{Kd} \)
    \item[-] \(\bar H_t(\theta) := \sum_{s=1}^t \nabla\mu(\theta x_s) \otimes x_s x_s^\top + \sum_{s=\tau+1}^{t-1} 1_K 1_K^\top \otimes x_s x_s^\top + \lambda_0 I_{Kd} \)
    \item[-] \(g_t(\theta) := \sum_{s=1}^t \mu(\theta x_s) \otimes x_s + \lambda_0 \theta \)
    \item[-] \(G_t(\theta_1, \theta_2) := \sum_{s=1}^t \int_0^1 \nabla \mu((v \theta_1 + (1-v)\theta_2)x_s)dv \otimes x_s x_s^\top + \lambda_0 I_{Kd} \)
    \item[-] \(g_t(\theta_1) - g_t(\theta_2) = G_t(\theta_1, \theta_2) (\theta_1 - \theta_2) \qquad\text{(Mean-value Theorem)} \)
    \item[-] \(W_t := \sum_{s=\tau+1}^{t-1} \nabla \mu(\theta_{s+1} x_s) \otimes x_s x_s^\top + \lambda I_{Kd} \)
    \item[-] \(\bar W_t := \sum_{s=\tau+1}^{t-1} \nabla \mu(\theta_{s+1} x_s) \otimes x_s x_s^\top + \sum_{s=\tau+1}^{t-1} 1_K 1_K^\top \otimes x_s x_s^\top + \lambda I_{Kd} \)
    \item[-] \(\alpha_s(\theta_1, \theta_2) := \int_0^1 \nabla \mu(((1-v) \theta_1 + v\theta_2)x_s)dv \otimes x_s x_s^\top \)
    \item[-] \(\alpha_s(\theta_1, \theta_2) = \alpha_s(\theta_2, \theta_1) \qquad\text{(change of variable)} \)
    \item[-] \(\tilde \alpha_s(\theta_1, \theta_2) := \int_0^1 (1-v) \nabla \mu(((1-v) \theta_1 + v\theta_2)x_s)dv \otimes x_s x_s^\top \)
    \item[-] \(\alpha(\theta_1 x_1, \theta_2 x_2) := \int_0^1 \nabla \mu((1-v)\theta_1 x_1 + v \theta_2 x_2 )dv \)
\end{itemize}



\section{Bounds on the Constants \texorpdfstring{$\kappa$}{kappa} and  \texorpdfstring{$\kappa_*$}{kappa\_*}}\label{appendix:bounds on constants kappa}
\subsection{Upper and lower bounds on the constant \texorpdfstring{$\kappa$}{kappa}}\label{appendix:bounding kappa}

In this appendix, we show the following lemma that bounds $\kappa$ by above and by below. In particular, we recover up to constant factors the bounds proved by \cite{amani2021ucb} for earlier definitions of $\kappa$ (see Appendix~\ref{appendix:previous framework} thereafter).  

\begin{restatable}{lmmnonumber}{} For any even $K \in \mathbb{N}$ and $\mathcal{X} = \{x \in \R^d: \|x\| \leq X\}$, we have
\[
      \frac{K}{4} + \frac{K}{4} e^{2SX} \leq \kappa \leq K e^{2SX} \,.
\]
for $\kappa$ as defined in Equation~\eqref{eq:kappa}. 
\end{restatable}

\begin{proof}
We first prove the upper-bound. Fix any \(z\in\R^K\). We bound the second smallest eigenvalue \(\lambda_{K-1}\) of \(\nabla \mu(z)\) using Weyl's inequality \citep{weyl1912asymptotische} and the definition of the gradient of the softmax \( \nabla\mu(z) = \diag(\mu(z)) - \mu(z)\mu(z)^\top \). A direct application of Weyl's inequality gives
\begin{equation*}
    \lambda_{K-1}(\diag(\mu(z)) - \mu(z)\mu(z)^\top) \ge \lambda_K(\diag(\mu(z))) + \lambda_{K-1}( - \mu(z)\mu(z)^\top) = \min_{i\in\llbracket K \rrbracket} \mu(z)_i \,.
\end{equation*}
Thus we have 
\begin{equation*}
    \lambda_{K-1}(\diag(\mu(z)) - \mu(z)\mu(z)^\top) \ge \dfrac{\exp(-S X)}{K \exp(SX)} = \dfrac{1}{K} \exp(-2 SX)
\end{equation*}
where \( X := \max_{x\in\X} \lVert x \rVert_2 \) and \(S\) is assumed such that \( \lVert \theta \rVert_2 \le S \). Hence,
\[
    \kappa := \frac{1}{ \min_{\lVert\theta\rVert_2\le S} \min_{x\in\X} \lambda_{K-1}(\nabla\mu(\theta x))}  \leq K e^{2SX} \,.
\]

We now prove the lower-bound. For simplicity, we assumed that \(\X\) is a ball of radius \(X\) and that \(K\) is even.
A direct application of the Schur-Horn Theorem gives for all \(z\in\R^K\)
\begin{equation*}
    \min_{i,j,i\neq j} \nabla\mu(z)_{ii}+\nabla\mu(z)_{jj} \ge \lambda_{K}(\nabla\mu(z)) + \lambda_{K-1}(\nabla\mu(z)) = \lambda_{K-1}(\nabla\mu(z)) \,.
\end{equation*}
We choose \( \theta \) such that \( \lVert \theta \rVert_2 = S \) and with the first \(K/2\) rows equal to each other, i.e. \( [\theta]_1 = [\theta]_i \) for \(i\in\llbracket K/2 \rrbracket\) and with the others rows collinear in the opposite direction, i.e. \( [\theta]_i = -[\theta]_1 \) for all \(i\ge 2\).
We choose \(x\) such that \( x = -\tfrac{[\theta]_1}{S} X \). Thus we obtain
\begin{multline*}
    \dfrac{4}{K\big(1 + \exp(2SX)\big)} \ge \dfrac{2\exp(-SX)}{\tfrac{K}{2} \exp(-SX) + \tfrac{K}{2} \exp(SX) } \ge \min_{x\in\X}\min_{\lVert\theta\rVert_2\le S} \min_{i,j,i\neq j} \mu(\theta x)_{ii}+ \mu(\theta x)_{jj} \\
    \ge \min_{x\in\X}\min_{\lVert\theta\rVert_2\le S} \min_{i,j,i\neq j} \nabla\mu(\theta x)_{ii}+\nabla\mu(\theta x)_{jj} \ge \min_{x\in\X}\min_{\lVert\theta\rVert_2\le S} \lambda_{K-1}(\nabla\mu(\theta x)) =: \kappa^{-1} \,,
\end{multline*}
which concludes the proof. 
\end{proof}

\subsection{Example of large \texorpdfstring{\(\kappa_*\)}{kappa\_*}}\label{appendix:large kappa star}

Let \(K\ge 2\) be even and \(d\ge 1\). Let us consider the following problem, we define \( \theta_* = \Pi M_* \in \Pi\R^{K\times d} \) with \(M_*\) equal to
\begin{equation*}
M_* := 
\begin{bmatrix}
    m & 0 & \dots & 0 \\
    0 & 0 & \dots & 0 \\
    \vdots & \vdots & \ddots & \vdots \\
    0 & 0 & \dots & 0
\end{bmatrix} \in \R^d 
\end{equation*}
where \(m>0\), moreover \(M_*\) is such that \( \lVert \theta_* \rVert_2 = S \). We define \(\rho\in\R^K\) such that
\begin{equation*}
\rho := \dfrac{1}{\sqrt{K+3}}
\begin{bmatrix}
    2 \\
    1\\
    \vdots \\
    1
\end{bmatrix} \,.
\end{equation*}
Note that $\lVert \rho \rVert_2 = 1$. We choose \( \X = \mathcal{S}_1(\R^d) \). We have that \( x_* = [M_*]_1/\lVert [M_*]_1 \rVert_2 \). Note that \( \lVert x_* \rVert_2 = 1 = X \). Let us compute \( \kappa_* \):
\begin{equation*}
    \kappa_*^{-1} = \rho^\top \nabla \mu(\theta_* x_*) \rho = \rho^\top \left( \mathrm{diag}(\mu(\theta_* x_*)) - \mu(\theta_* x_*) \mu(\theta_* x_*)^\top \right) \rho
\end{equation*}
where the second equality is due to the definition of $\nabla\mu(\cdot)$. This can be developed into:
\begin{align}
    \kappa_*^{-1} &= \sumK \rho_k \mu(\theta_* x_*)_k \left[ \sum_{i=1}^K \rho_i \left( \delta_{ik} - \mu(\theta_* x_*)_i \right) \right] \nonumber \\
    &= \rho_1^2 \mu(\theta_* x_*)_1 (1-\mu(\theta_* x_*)_1) -2 \rho_1 \mu(\theta_* x_*)_1 \sum_{k=2}^K \rho_k \mu(\theta_* x_*)_k \nonumber  \\
    &\quad + \sum_{k=2}^K \rho_k \mu(\theta_* x_*)_k \left[ \sum_{i=2}^K \rho_i \left( \delta_{ik} - \mu(\theta_* x_*)_i \right) \right] \label{eq:large kappa star eq1} \,.
\end{align}
Let us prove that the first two terms cancel each other.
\begin{align*}
    \rho_1^2 \mu(\theta_* x_*)_1 (1-\mu(\theta_* x_*)_1) &= \rho_1 2\rho_2 \mu(\theta_* x_*)_1 (1-\mu(\theta_* x_*)_1) &&(\rho_1 = 2\rho_2) \\
    &= 2\rho_1 \rho_2 \mu(\theta_* x_*)_1 \sum_{k=2}^K \mu(\theta_* x_*)_k &&(\mu \text{ is a probability)} \\
    &= 2\rho_1 \mu(\theta_* x_*)_1 \sum_{k=2}^K \rho_k \mu(\theta_* x_*)_k &&(\rho_2 = \rho_k ,\forall k \in \llbracket 2, K \rrbracket )
\end{align*}
Consequently, Equation~\eqref{eq:large kappa star eq1} becomes
\begin{align*}
    \kappa_*^{-1} &= \sum_{k=2}^K \rho_k \mu(\theta_* x_*)_k \left[ \sum_{i=2}^K \rho_i \left( \delta_{ik} - \mu(\theta_* x_*)_i \right) \right] \\
    &= \dfrac{2}{K+3} \sum_{k=2}^K \mu(\theta_* x_*)_k \left[ \sum_{i=2}^K \left( \delta_{ik} - \mu(\theta_* x_*)_i \right) \right] &&(\text{Def of }\rho) \\
    &= \dfrac{2}{K+3} \sum_{k=2}^K \mu(\theta_* x_*)_k \left[ 1 - \sum_{i=2}^K \mu(\theta_* x_*)_i \right] \\
    &= \dfrac{2}{K+3} \sum_{k=2}^K \mu(\theta_* x_*)_k \mu(\theta_* x_*)_1 \\
    &\le 2 \mu(\theta_* x_*)_2 \mu(\theta_* x_*)_1 \,.
\end{align*}
We now use the definition of the softmax to upper-bound the probabilities.
\begin{align*}
    \kappa_*^{-1} &\le 2 \dfrac{1}{K-1 + \exp([M_*]_1^\top x_*) } \cdot \dfrac{\exp([M_*]_1^\top x_*)}{K-1 + \exp([M_*]_1^\top x_*)} \\
    &= 2\dfrac{\exp(-[M_*]_1^\top x_*)}{(K-1)\exp(-[M_*]_1^\top x_*) + 1} \cdot \dfrac{1}{(K-1)\exp(-[M_*]_1^\top x_*) + 1} \\
    &= 2\dfrac{\exp(-S)}{(K-1)\exp(-S) + 1} \cdot \dfrac{1}{(K-1)\exp(-S) + 1} \\
    &\le 2\exp(-S)
\end{align*}
We exhibit a case where $\kappa_*$ is exponentially small in $S = \lVert \theta_* \rVert_2$.
In this case, by Theorem~\ref{thm:regret bound}, the asymptotic regret is thus of order 
\[
    \Reg_T \leq \tilde O\Big( R d \exp(-S/2) \sqrt{KT} \Big)\,.
\]

\section{Linking \texorpdfstring{$\kappa_*$}{kappa*} with the Fisher Information}\label{appendix:fisher information}

In this section we show that $\kappa_*^{-1}$ corresponds to the Cramér-Rao lower-bound:
\begin{equation*}
    \dfrac{R^2}{\kappa_*} := \rho^\top \nabla \mu(\theta_* x_*) \rho = \mathrm{CRLB}(\theta_* x_*)
\end{equation*}
using the Cramér-Rao bound on the function $\Gamma:\R^{K} \to \R$ defined as
\begin{equation*}
    \Gamma(z) = \rho^\top \mu(z) \,.
\end{equation*}
Let $Y\in\R^K$ be a random variable following the law $Y\sim\textrm{Categorical}(\mu(z))$, $Y$ is a one-hot vector. Its log-likelihood is defined as $ \ell(Y) = \sumK Y_k \log(\mu(z)_k) $. 
The Cramér-Rao bound states that:
\begin{equation}
    \mathrm{CRLB}(z) = \nabla \Gamma(z)^\top I_\ell(z)^\dag \nabla\Gamma(z) \label{eq:cramer rao statement}
\end{equation}
where $I_\ell(z)^\dag$ denotes the Moore-Penrose pseudo-inverse of the Fisher information of $\ell$. We have
\begin{multline*}
    I_\ell(z) = \E\left[ (\nabla_z \ell) (\nabla_z \ell)^\top \right] = \E\left[ (Y-\mu(z)) (Y-\mu(z))^\top \right] \\
    = \E\left[ YY^\top - Y\mu(z)^\top - \mu(z)Y^\top + \mu(z)\mu(z)^\top \right] 
    = \E\left[YY^\top \right] - \mu(z)\mu(z)^\top = \diag(\mu(z)) - \mu(z)\mu(z)^\top \,.
\end{multline*}
Let us now compute $\nabla_z\Gamma$, we have:
\begin{equation*}
    \nabla\Gamma(z) = \left( \diag(\mu(z)) - \mu(z)\mu(z)^\top \right) \rho \,.
\end{equation*}
We substitute into Equation~\eqref{eq:cramer rao statement} and get
\begin{align*}
    \mathrm{CRLB}(z) &= \rho^\top \left( \diag(\mu(z)) - \mu(z)\mu(z)^\top \right)^\top \left(\diag(\mu(z)) - \mu(z)\mu(z)^\top \right)^\dag \left( \diag(\mu(z)) - \mu(z)\mu(z)^\top \right) \rho \\
    &= \rho^\top \left( \diag(\mu(z)) - \mu(z)\mu(z)^\top \right)^\top \rho \\
    &= \rho^\top \left( \diag(\mu(z)) - \mu(z)\mu(z)^\top \right) \rho \,.
\end{align*}
Thus using the definition of the gradient of the softmax and choosing $z=\theta_* x_*$ we have 
\begin{equation*}
    \mathrm{CRLB}(\theta_* x_*) = \rho^\top \nabla\mu(\theta_* x_*) \rho =: \dfrac{R^2}{\kappa_*} \,.
\end{equation*}

\section{Comparison with the Framework of \texorpdfstring{\citet{amani2021ucb}}{Amani and Thrampoulidis (2021)}}\label{appendix:previous framework}

\citet{amani2021ucb} also consider a MNL bandit framework, which is equivalent but defined slightly differently from ours. 
In their framework, the environment parameter \(\tilde \theta_* \in \R^{K\times d} \) is defined with its last row equal to zero \( [\tilde \theta_*]_K = 0_d \). Therefore the probability of a decision \(i\in\llbracket K \rrbracket\) becomes
\begin{equation*}
    \pr[y_t = i | x_t] = 
    \begin{cases}
        \dfrac{1}{1 + \sum_{k=1}^{K-1} \exp([\tilde\theta_*]_k x_t) } \quad\text{if } i=K \\
        \dfrac{\exp([\tilde \theta_*]_ix_t)}{1+ \sum_{k=1}^{K-1} \exp([\tilde\theta_*]_k x_t) } \quad\text{if } i< K
    \end{cases} \,.
\end{equation*}
The reward vector is also defined \(\tilde \rho\in\R^K_+\) but with its last element equal to zero \( \rho_K = 0 \). The regret is defined as
\begin{equation*}
    \widetilde \Reg_T = \sum_{t=1}^T \sum_{k=1}^{K-1} \tilde\rho_k \left( \pr[y_t = k | x_*] - \pr[y_t = k | x_t] \right) \,.
\end{equation*}
Thus the last element of the probability vector is not needed and we define the vector \( \tilde \mu(\theta x) \in \R^{K-1} \) as the truncated probability vector \( [\tilde \mu(\theta x)]_k = \pr[y_t = k | x_t] \). Contrary to our case, the fact that \(\tilde\mu\) is not a probability ensures that its minimum eigenvalue is well-defined \citep[Lemma~5]{amani2021ucb}. The problem-dependent constant measuring the non-linearity is defined as:
\begin{equation*}
    \tilde \kappa := \dfrac{1}{\min_{x\in\X} \min_{\lVert \theta \rVert_2\le S} \lambda_{\min}(\diag(\tilde\mu(\theta x)) - \tilde\mu(\theta x)\tilde\mu(\theta x)^\top)} \,.
\end{equation*}
As shown by \citep[Eq.~(20)]{amani2021ucb} the constant \(\tilde\kappa\) is exponentially large with respect to \(S\) and \(X\). These lower and upper bounds on \(\tilde\kappa\) show that our constant \(\kappa\) is comparable, see Appendix~\ref{appendix:bounding kappa}.

Now note that in our framework, by choosing without loss of generality \(\min_k \rho_k = \rho_K\) we have
\begin{align*}
    \Reg_T &:= \sum_{t=1}^T \rho^\top (\mu(\theta_* x_*) - \mu(\theta_* x_t)) \\
    &= \sum_{t=1}^T (\rho - \rho_K 1_K)^\top (\mu(\theta_* x_*) - \mu(\theta_* x_t)) \\
    &+ \sum_{t=1}^T \rho_K 1_K^\top (\mu(\theta_* x_*) - \mu(\theta_* x_t)) \\
    &= \sum_{t=1}^T (\rho - \rho_K 1_K)^\top (\mu(\theta_* x_*) - \mu(\theta_* x_t)) \\
    &= \sum_{t=1}^T \sum_{k=1}^{K-1} (\rho_k - \rho_K ) (\mu(\theta_* x_*)_k - \mu(\theta_* x_t)_k)
\end{align*}
We could then choose an arbitrary value for \([\theta_*]_K\). For \( [\theta_*]_K = 0_d \) we recover the framework of \citet{amani2021ucb}. Thus their framework is included in ours.

\section{Discussion of the Multinomial \textit{Logit} Bandits}\label{appendix:logit bandits}

In this section we discuss the differences between Multinomial Logistic Bandits, our framework, and the Multinomial \textit{Logit} Bandit framework. In our setting, the environment may have multiple reactions to a single action. On the other hand, in the \textit{Logit} setting the agent selects a set of items to which to environment responds with either a click or no click. In our problem setting, the probability of observing decision $y_t=k$ given context $x_t$ is: 
\begin{equation*}
    \pr[y_t=k |x_t] = \dfrac{\exp((\theta_*)_k^\top x_t)}{\sum_{i=1}^K \exp((\theta_*)_i^\top x_t)} 
\end{equation*}
where each possible decision $i\in\llbracket K \rrbracket$ has its own parameter $(\theta_*)_i \in \R^d$. Thus, we are estimating a different parameter vector for each decision, and the variation in decision probabilities comes from these parameter differences.

In contrast, in the \textit{Logit} setting, the agent chooses a subset $S_t\subseteq \llbracket K \rrbracket$, and the environment responds with a choice over that subset. The probability of observing decision $k\in S_t$ is: 
\begin{equation*}
    \pr[y_t=k|S_t] = \dfrac{\exp(\theta_*^\top x_{t,k})}{\sum_{i\in S_t} \exp(\theta_*^\top x_{t,i})}
\end{equation*}
where the agent observes the context vectors $x_{t,i}$ for all $i\in\llbracket K \rrbracket$, and there is a single share parameter $\theta_* \in \R^d$. In this setting, all variations in decision probabilities arise from the context vectors, not from the parameter $\theta_*$.

In summary, the settings differ fundamentally in their parameterisation, feedback structure, and modeling assumptions. Ours involves learning distinct models per action; the \textit{Logit} setting uses a shared parameter across all items and focuses on contextual differences.

\section{Analysis of Algorithm~\ref{algo:learning routine}}\label{appendix:analysis algorithm}

\subsection{Exploration Routine}\label{appendix: exploration routine}

\subsubsection{Confidence Set}
We build our confidence set over this proposition from \citet[Theorem~1]{zhang2024online}, which is itself an improvement of \citet[Theorem~1]{amani2021ucb}. As demonstrated by \citet[Section~6]{abeille2021instance}, the confidence set presented in the following proposition is not convex. To address this, we construct a convex relaxation, see Proposition~\ref{prop:confidence set after exploring}.

\begin{restatable}{prop}{PropIntermediaryCondidenceSet}\label{prop:exploration first confidence lemma}
    Set the parameter \( \lambda_0 = (S+1) Kd \log(T/\delta) \) with a certain \(\delta \in (0, 1]\). Let the event \(E_\delta\) be defined by 
    \begin{equation*}
        E_\delta : \{\forall t \ge 1, \lVert g_t(\theta_*) - g_t(\hat \theta_{t+1}) \rVert^2_{H_t^{-1}(\theta_*)} \le \gamma_t(\delta) \}
    \end{equation*}
    where \( \gamma_t(\delta) := 16 \lambda_0 \).
    We have that 
    \begin{equation*}
        \pr (E_\delta) \ge 1-\delta\,.
    \end{equation*}
\end{restatable}

Note that \( V_t \preccurlyeq \bar H_t(\theta_*) \) for \( \bar H_t(\theta) := H_t(\theta) + \sum_{s=1}^t 1_K 1_K^\top \otimes x_s x_s^\top \), therefore proving the following lemma is sufficient to prove that 
\begin{equation*}
    \pr( \theta_* \in \Theta) \ge 1 - \delta \,.
\end{equation*}

\begin{restatable}{prop}{PropConfidenceSetAfterExploring}\label{prop:confidence set after exploring}
    Let \(\delta\in(0,1]\) and \(\hat \theta_{t+1}\) be defined as in Algorithm~\ref{algo:exploration routine}. We have that 
    \begin{equation*}
        \pr\left(\forall t\ge 1, \lVert\hat\theta_{t+1} - \theta_* \rVert^2_{\bar H_t(\theta_*)} \le \beta_t(\delta)\right) \ge 1-\delta 
    \end{equation*}
    where \( \bar H_t(\theta) := H_t(\theta) + \sum_{s=1}^t 1_K 1_K^\top \otimes x_s x_s^\top \) and \(\beta_t(\delta) := \left(1 + \dfrac{\gamma_t(\delta)}{\lambda_0} + \sqrt{\dfrac{\gamma_t(\delta)}{\lambda_0}} \right)^2 \gamma_t(\delta)\) with \(\gamma_t(\delta)\) and \(\lambda_0\) defined in Proposition~\ref{prop:exploration first confidence lemma}.
\end{restatable}

\begin{proof}
We follow the proof of Lemma~1 in \citet{faury2022jointly}.

\smallskip
\noindent\textbf{Step 1: Sub-Exponential Self-Concordance.}

We first show that for all time step \(t\ge 1\), if the event \(E_\delta\) holds, we have that
\begin{equation*}
    H_t(\theta_*) \preccurlyeq \left(1 + \dfrac{\gamma_t(\delta)}{\lambda_0} + \sqrt{\dfrac{\gamma_t(\delta)}{\lambda_0}} \right) G_t(\theta_*, \hat \theta_{t+1})
\end{equation*}
where \(\lambda_0\) and \(\gamma_t(\delta)\) are defined in Proposition~\ref{prop:exploration first confidence lemma}.
From the proof of Lemma~13 in \citet{amani2021ucb} we have that
\begin{equation*}
    H_t(\theta_*) \preccurlyeq \sum_{s=1}^t (1+d(x_s, \hat \theta_{t+1}, \theta_*)) \alpha_s(\hat \theta_{t+1}, \theta_*) + \lambda_0 I_{Kd}
\end{equation*}
where \( d(x_s, \hat \theta_{t+1}, \theta_*) := \lVert (\hat\theta_{t+1} - \theta_*)x_s\rVert_2 \). From now on the proof of Lemma~2 of \citet{abeille2021instance} also holds in the multiclass setting to conclude this proof step. We provide it for the sake of completeness. We apply Cauchy-Schwarz inequality and obtain
\begin{multline*}
    d(x_s, \hat \theta_{t+1}, \theta_*) \le \lVert x_s \rVert_{G_t^{-1}(\hat \theta_{t+1}, \theta_*)} \lVert \hat\theta_{t+1} - \theta_* \rVert_{G_t(\hat \theta_{t+1}, \theta_*)} \\ \le \lVert x_s \rVert_{G_t^{-1}(\hat \theta_{t+1}, \theta_*)} \lVert g_t(\hat\theta_{t+1}) - g_t(\theta_*) \rVert_{G_t^{-1}(\hat \theta_{t+1}, \theta_*)} \le \lambda_0^{-1/2} \lVert g_t(\hat\theta_{t+1}) - g_t(\theta_*) \rVert_{G_t^{-1}(\hat \theta_{t+1}, \theta_*)} \,.
\end{multline*}
Putting it back we get
\begin{equation}\label{eq:sub_exp_self_concordance eq1}
    H_t(\theta_*) \preccurlyeq \left(1+ \lambda_0^{-1/2} \lVert g_t(\hat\theta_{t+1}) - g_t(\theta_*) \rVert_{G_t^{-1}(\hat \theta_{t+1}, \theta_*)} \right) G_t(\theta_*, \hat\theta_{t+1})\,.
\end{equation}
Therefore using this matrix inequality and event \(E_\delta\) we get
\begin{align*}
    \lVert g_t(\hat \theta_{t+1}) &- g_t(\theta_*) \rVert^2_{G_t^{-1}(\theta_*, \hat\theta_{t+1})} \\
    &\le \left( 1+ \lambda_0^{-1/2} \lVert g_t(\hat\theta_{t+1}) - g_t(\theta_*) \rVert_{G_t^{-1}(\hat \theta_{t+1}, \theta_*)} \right) \lVert g_t(\hat \theta_{t+1}) - g_t(\theta_*) \rVert^2_{H_t^{-1}(\theta_*)} \\
    &\le \gamma_t(\delta) \lambda_0^{-1/2} \lVert g_t(\hat\theta_{t+1}) - g_t(\theta_*) \rVert_{G_t^{-1}(\hat \theta_{t+1}, \theta_*)} + \gamma_t(\delta) \,.
\end{align*}
Solving for \(\lVert g_t(\hat\theta_{t+1}) - g_t(\theta_*) \rVert_{G_t^{-1}(\hat \theta_{t+1}, \theta_*)}\) we get 
\begin{equation*}
    \lVert g_t(\hat\theta_{t+1}) - g_t(\theta_*) \rVert_{G_t^{-1}(\hat \theta_{t+1}, \theta_*)} \le \gamma_t(\delta) \lambda_0^{-1/2} + \sqrt{\gamma_t(\delta)}\,.
\end{equation*}
We now put this back in Equation\eqref{eq:sub_exp_self_concordance eq1} to conclude and obtain:
\begin{equation*}
    H_t(\theta_*) \preccurlyeq \left(1 + \dfrac{\gamma_t(\delta)}{\lambda_0} + \sqrt{\dfrac{\gamma_t(\delta)}{\lambda_0}} \right) G_t(\theta_*, \hat \theta_{t+1}) \,.
\end{equation*}

\medskip\noindent
\textbf{Step 2: Applying Self-concordance.}

We apply twice the self-concordance property to get
\begin{align*}
    \lVert \theta_* - \hat\theta_{t+1}\rVert^2_{H_t(\theta_*)} &\le \left(1 + \dfrac{\gamma_t(\delta)}{\lambda_0} + \sqrt{\dfrac{\gamma_t(\delta)}{\lambda_0}} \right) \lVert \theta_* - \hat\theta_{t+1}\rVert^2_{G_t(\theta_*, \hat\theta_{t+1})} \\
    &\le \left(1 + \dfrac{\gamma_t(\delta)}{\lambda_0} + \sqrt{\dfrac{\gamma_t(\delta)}{\lambda_0}} \right) \lVert g_t(\theta_*) - g_t(\hat\theta_{t+1}) \rVert^2_{G_t^{-1}(\theta_*, \hat\theta_{t+1})} \\
    &\le \left(1 + \dfrac{\gamma_t(\delta)}{\lambda_0} + \sqrt{\dfrac{\gamma_t(\delta)}{\lambda_0}} \right)^2 \lVert g_t(\theta_*) - g_t(\hat\theta_{t+1}) \rVert^2_{H_t^{-1}(\theta_*)} \\
    &\le \left(1 + \dfrac{\gamma_t(\delta)}{\lambda_0} + \sqrt{\dfrac{\gamma_t(\delta)}{\lambda_0}} \right)^2 \gamma_t(\delta) \\
    &=: \beta_t(\delta) \,.
\end{align*}

\medskip\noindent
\textbf{Step 3: From \(H_t(\theta_*)\) to \(\bar H_t(\theta_*)\).}

We decompose \(\R^K\) as \( \R^K = 1_K \oplus \calH \) where \(\calH\) is the hyperplane supported by \(1_K\). Recall that \(\theta_* \in \Pi R^{K\times d} \) and \(\hat \theta_{t+1}\in\Pi\R^{K\times d} \), for all \(x\in \X\), by definition of \(\Pi\), \( \theta_* x \) and \( \hat \theta_{t+1} x \) are in \(\calH\). Therefore \( \sum_{s=1}^t \lVert (\theta_* - \hat \theta_{t+1})x_s \rVert^2_{1_K 1_K^\top} = 0 \). And we can conclude by
\begin{equation*}
    \lVert \theta_* - \hat\theta_{t+1}\rVert^2_{\bar H_t(\theta_*)} = \lVert \theta_* - \hat\theta_{t+1}\rVert^2_{H_t(\theta_*)} \le \beta_t(\delta) \,.
\end{equation*}

\end{proof}

\subsubsection{Proof of Lemma~\ref{lemma:constant diameter}}\label{appendix:proof constant diameter}
\LmmConstantDiameter*
We adapt the proof of Lemma~2 of \citet{faury2022jointly} to the multiclass setting.
\begin{proof}
We start by making the term \(I_K\) appear in order to match the dimension of \(\bar H_t(\theta_*)\).
\begin{align*}
    \diam(\Theta) &= \max_{x\in\X}\max_{\theta_1, \theta_2 \in \Theta} \lVert (\theta_1 - \theta_2) x \rVert_2 \\
    &= \max_{x\in\X}\max_{\theta_1, \theta_2 \in \Theta} \lVert (I_K \otimes x^\top ) (\theta_1 - \theta_2) \rVert_2 \\
    &\le \max_{x\in\X}\max_{\theta_1, \theta_2 \in \Theta} \lVert I_K \otimes x^\top \rVert_{V_\tau^{-1}} \lVert \theta_1 - \theta_2\rVert_{V_\tau} && \text{Cauchy-Schwarz} \\
    &\le 2 \sqrt{\beta_\tau(\delta)} \max_{x\in\X} \lVert I_K \otimes x \rVert_{V_\tau^{-1}} && \text{Proposition~\ref{prop:confidence set after exploring} and symmetry} \\
    &= 2 \sqrt{\beta_\tau(\delta)} \sqrt{\max_{x\in\X} \lVert I_K \otimes x \rVert^2_{V_\tau^{-1}}} \\
    &= 2 \sqrt{\beta_\tau(\delta)} \tau^{-1/2} \sqrt{\sum_{s=1}^\tau \max_{x\in\X} \lVert I_K \otimes x \rVert^2_{V_\tau^{-1}}} \\
    &\le 2 \sqrt{\beta_\tau(\delta)} \tau^{-1/2} \sqrt{\sum_{s=1}^\tau \max_{x\in\X} \lVert I_K \otimes x \rVert^2_{V_{s-1}^{-1}}} && (V_\tau \succcurlyeq V_{s-1}) \\
    &\le 2 \sqrt{\beta_\tau(\delta)} \tau^{-1/2} \sqrt{\sum_{s=1}^\tau \lVert I_K \otimes x_s \rVert^2_{V_{s-1}^{-1}}} && \text{definition of \(x_s\)} \\
    &=2 \sqrt{\beta_\tau(\delta)} \tau^{-1/2} \kappa^{1/2} \sqrt{\sum_{s=1}^\tau \lVert \kappa^{-1/2} I_K \otimes x_s \rVert^2_{V_{s-1}^{-1}}} \\
    &\le4 \sqrt{\beta_\tau(\delta)} \tau^{-1/2} \kappa^{1/2} \sqrt{Kd \log\left(1+\tfrac{\tau}{Kd}\right)} && \text{\citet[lemma~10]{abbasi2011improved}}
\end{align*}
Thus if we choose \(\tau = 96 \beta_\tau(\delta)\kappa Kd\log\left(1+\tfrac{T}{Kd}\right) \) we have that \( \diam(\Theta) \le 1/\sqrt{6} \).
\end{proof}

\subsection{Proof of Lemma~\ref{lemma:confidence set learning}}\label{appendix:proof confidence set learning}

\LmmConfidenceSetLearning*

\begin{proof}
First, by Lemma~\ref{lemma:constant diameter}, we can apply \citep[Theorem~4.2]{lee2025improved} with $\alpha=1/\sqrt{6}$ to obtain
\begin{equation*}
    \lVert \theta_* - \theta_{t+1} \rVert_{W_{t+1}} \le \sigma_{t}(\delta)
\end{equation*}
with probability $1-\delta$. We then decompose \(\R^K\) as \( \R^K = 1_K \oplus \cH \) where \(\cH\) is the hyperplane supported by \(1_K\). 
Recall that \(\theta_\ast, \theta_{t+1} \in \Pi \R^{K\times d} \), for all \(x\in\cX\), 
by definition of \(\Pi\), \(\theta_{t+1} x\) and \( \theta_\ast x \) are in \(\cH\). 
Therefore \(  \sum_{s=1}^t \lVert (\theta_{t+1} - \theta_\ast)x_s \rVert_{1_K 1_K^\top} = 0 \). 
And we conclude that with probability \(1-2\delta\)
\begin{equation*}
    \lVert \theta_{t+1} - \theta_\ast \rVert_{\bar W_{t+1}} = \lVert \theta_{t+1} - \theta_\ast \rVert_{W_{t+1}} \le \sigma_t(\delta) \,.
\end{equation*}

\end{proof}

\subsection{Proof of Theorem~\ref{thm:regret bound}}\label{appendix:proof thm regret bound}
\ThmRegretBound*
\begin{proof}

\medskip\noindent 
Throughout the proof we assume that
\begin{equation*}
    \diam(\Theta) \le 1 \qquad \text{ and } \qquad 
    \forall t\ge1, \quad \lVert \theta_* - \theta'_{t+1} \rVert_{\bar W_{t+1}} \le \sigma_t(\delta)
\end{equation*}
which is verified with probability \(1-2\delta\) thanks to Lemma~\ref{lemma:constant diameter} and Lemma~\ref{lemma:confidence set learning}, we apply a Union Bound at the end. 
Moreover we choose an optimal action $x_*$ which verifies
\begin{equation*}
    x_* \in \left( \argmax_{x\in\cX} \rho^\top \mu(\theta_* x) \right) \bigcap \left( \argmax_{x\in\cX} (\nu \rho + \rho^{\odot 2})^\top \mu(\theta_* x) \right) \,.
\end{equation*}
The existence of such an action $x_*$ is guaranteed by Lemma~\ref{lemma:x_* exists}.

\smallskip
The regret of the exploration phase is smaller than
\begin{equation*}
    \tau \le \cste K^{3/2} d^{3/2} \kappa \log^{3/2}(T) \,.
\end{equation*}
Let us now focus on the second phase of the algorithm.

\medskip\noindent
\textbf{Step 1: Using optimism.}
Using the definition of the optimistic reward we can bound the regret twice.
\begin{align*}
    \Reg_T(\mathrm{Learning}) &:= \sum_{t=\tau+1}^T \rho^\top (\mu(\theta_*x_*) - \mu(\theta_* x_t)) \\
    &\le \sum_{t=\tau+1}^T \rho^\top \mu(\theta'_t x_*) + \epsilon_{1,t}(x_*) + \epsilon_{2,t}(x_*) - \rho^\top \mu(\theta_* x_t) && \text{(Prop.~\ref{prop:optimistic reward})} \\
    &\le \sum_{t=\tau+1}^T \rho^\top \mu(\theta'_t x_t) + \epsilon_{1,t}(x_t) + \epsilon_{2,t}(x_t) - \rho^\top \mu(\theta_* x_t) &&\text{(Def. of \(x_t\))} \\
    &\le 2 \sum_{t=\tau+1}^T \epsilon_{1,t}(x_t) + 2 \sum_{t=\tau+1}^T \epsilon_{2,t}(x_t) &&\text{(Prop.~\ref{prop:optimistic reward})} \\
    &\le 2 \sum_{t=1}^T \epsilon_{1,t}(x_t) + 2 \sum_{t=1}^T \epsilon_{2,t}(x_t)
\end{align*}

\textbf{Step 2: Bounding the sum of \(\epsilon_{2,t}(x_t) \).}
We start by bounding the second sum
\begin{align*}
    \sum_{t=1}^T \epsilon_{2,t}(x_t) &= 3 \sum_{t=1}^T R \sigma_t(\delta)^2 \lVert (I_K \otimes x_t^\top) \bar W_t^{-1/2} \rVert_2^2 \\
    &\le 3 R \sigma_T(\delta)^2 \sum_{t=1}^T \lVert (I_K \otimes x_t^\top) \bar W_t^{-1/2} \rVert_2^2 \\
    &= 3 R \sigma_T(\delta)^2 \sum_{t=1}^T \lVert \bar W_t^{-1/2} (I_K \otimes x_t) \rVert_2^2 \\
    &\le 3 R \sigma_T(\delta)^2 \sum_{t=1}^T \lambda_{\max} \left( (I_K \otimes x_t^\top) \bar W_t^{-1} (I_K \otimes x_t) \right) \\ 
    &= 3 R \sigma_T(\delta)^2 \sum_{t=1}^T \lambda_{\max} \left( \left(I_K \otimes x_t x_t^\top\right) \bar W_t^{-1} \right) \\
    &\le 3 R \sigma_T(\delta)^2 \sum_{t=1}^T \Tr\left( \left(I_K \otimes x_t x_t^\top\right) \bar W_t^{-1} \right) \\
    &= 3 R \kappa \sigma_T(\delta)^2 \sum_{t=1}^T \Tr\left( \left( \tfrac{1}{\kappa} I_K \otimes x_t x_t^\top\right) \bar W_t^{-1} \right) \,.
\end{align*}
Let us define \( U_t := \sum_{s=1}^t \tfrac{1}{\kappa} I_K \otimes x_s x_s^\top +\tfrac{\lambda}{2} I_{Kd} \). We have that \( U_t \preccurlyeq \bar W_t \) for \(\lambda \ge 2\). We have that 
\begin{align*}
    \sum_{t=1}^T \epsilon_{2,t}(x_t) &\le 3 \kappa \sigma_T(\delta)^2 \sum_{t=1}^T \Tr\left( ( U_t - U_{t-1} ) U_t^{-1} \right) \\
    &\le 3 R \kappa \sigma_T(\delta)^2 \sum_{t=1}^T \log \dfrac{|U_t|}{|U_{t-1}|} &&\text{\citep[Lemma~4.5]{hazan2016introduction}} \\
    &\le 3 R \kappa \sigma_T(\delta)^2 Kd \log\left( 1+ \dfrac{T}{Kd\lambda\kappa} \right)\,. &&\text{(Lemma~\ref{lemma:det_trace inequality})}
\end{align*}

\medskip\noindent
\textbf{Step 3: Decomposing the sum of \(\epsilon_{1,t}(x_t) \).}
Let us now focus on the first sum.
\small
\begin{align*}
    \sum_{t=1}^T &\epsilon_{1,t}(x_t) \\
    &:= \sum_{t=1}^T \sigma_t(\delta) \lVert \bar W_t^{-1/2} (I_K \otimes x_t) \nabla\mu(\theta'_t x_t) \rho \rVert_2 \\
    &\le \sigma_T(\delta) \sum_{t=1}^T \lVert \bar W_t^{-1/2} (I_K \otimes x_t) \nabla\mu(\theta'_t x_t) \rho \rVert_2 \\
    &\le e \sigma_T(\delta) \sum_{t=1}^T \lVert \bar W_t^{-1/2} (I_K \otimes x_t) \nabla\mu(\theta'_t x_t)^{1/2} \nabla\mu(\theta_* x_t)^{1/2} \rho \rVert_2 &&\text{(Self-concordance)} \\
    &\le e \sigma_T(\delta) \sum_{t=1}^T \lVert \bar W_t^{-1/2} (I_K \otimes x_t) \nabla\mu(\theta'_t x_t)^{1/2} \rVert_2 \lVert \nabla\mu(\theta_* x_t)^{1/2} \rho \rVert_2 \\
    &\le e \sigma_T(\delta) \sqrt{\sum_{t=1}^T \lVert \bar W_t^{-1/2} (I_K \otimes x_t) \nabla\mu(\theta'_t x_t)^{1/2} \rVert_2^2} \sqrt{ \sum_{t=1}^T \lVert \nabla\mu(\theta_* x_t)^{1/2} \rho \rVert_2^2} \,. &&\text{(Cauchy-Schwarz)}
\end{align*}
\normalsize
Once again we have two separate terms to bound. We start with the left term.

\medskip \noindent
\textbf{Step 4: Bounding the sum of \(\lVert \bar W_t^{-1/2} (I_K \otimes x_t) \nabla\mu(\theta_{t+1} x_t)^{1/2} \rVert_2^2\).} \quad 
First, we lower-bound $\bar W_{t+1}$:
\begin{equation*}
    \bar W_{t+1} = \sum_{s=1}^{t-1} \nabla \mu(\theta_{s+1} x_s) \otimes x_s x_s^\top + \sum_{s=1}^{t-1} 1_K 1_K^\top \otimes x_s x_s^\top + \lambda I_{Kd}
    \succcurlyeq \sum_{s=1}^{t-1} 1_K 1_K^\top \otimes x_s x_s^\top + \lambda I_{Kd} \,.
\end{equation*}
We use the following equivalent of $\tfrac{1_K 1_K^\top}{K}$ in the Loewner order sense:
\begin{equation*}
    e_{11} \preccurlyeq \dfrac{1_K 1_K^\top}{K} \preccurlyeq e_{11} \in \R^{K\times K}
\end{equation*}
to obtain:
\begin{equation*}
    \bar W_{t+1} \succcurlyeq K \sum_{s=1}^{t-1} e_{11} \otimes x_s x_s^\top + \lambda I_{Kd}
    = K \sum_{s=1}^{t-1} e_{11}^2 \otimes x_s x_s^\top + \lambda I_{Kd} 
    = K \sum_{s=1}^{t-1} (e_{11} \otimes x_s) (e_{11} \otimes x_s)^\top + \lambda I_{Kd} \,.
\end{equation*}
Which is equivalent to
\begin{equation*}
    \bar W_{t+1} \succcurlyeq \sum_{s=1}^{t-1} \sumK (e_{kk} \otimes x_s) (e_{kk} \otimes x_s)^\top + \lambda I_{Kd} 
    = \sum_{s=1}^{t-1} I_K \otimes x_s x_s^\top + \lambda I_{Kd} \,.
\end{equation*}
Therefore we have
\begin{equation*}
    \lVert \bar W_t^{-1/2} (I_K \otimes x_t) \nabla\mu(\theta_{t+1} x_t)^{1/2} \rVert_2^2 \le 
    \llVert \left( \sum_{s=1}^{t-1} I_K \otimes x_s x_s^\top + \lambda I_{Kd} \right)^{-1/2} (I_K \otimes x_t) \nabla\mu(\theta_{t+1} x_t)^{1/2} \rrVert_2^2 \,.
\end{equation*}
We now use that $\nabla\mu(\theta_{t+1}x_t) \preccurlyeq I_K$ and get
\begin{equation*}
    \lVert \bar W_t^{-1/2} (I_K \otimes x_t) \nabla\mu(\theta_{t+1} x_t)^{1/2} \rVert_2^2  \le 
    \llVert \left( \sum_{s=1}^{t-1} I_K \otimes x_s x_s^\top + \lambda I_{Kd} \right)^{-1/2} (I_K \otimes x_t) \rrVert_2^2 \,.
\end{equation*}
We now upper-bound the sum over $T$ using a Trace-Determinant argument:
\begin{align*}
    \sumT \lVert \bar W_t^{-1/2} (I_K \otimes x_t) \nabla\mu(\theta_{t+1} x_t)^{1/2} \rVert_2^2  
    &\le \sumT \lambda_{\max} \left( (I_K \otimes x_t^\top)  \left( \sum_{s=1}^{t-1} I_K \otimes x_s x_s^\top + \lambda I_{Kd} \right)^{-1} (I_K \otimes x_s) \right) \\
    &= \sumT \lambda_{\max} \left( (I_K \otimes x_t x_t^\top)  \left( \sum_{s=1}^{t-1} I_K \otimes x_s x_s^\top + \lambda I_{Kd} \right)^{-1} \right) \\
    &\le \sumT \lambda_{\max} \left( x_t x_t^\top \left( \sum_{s=1}^{t-1} x_s x_s^\top + \lambda I_d  \right) \right) \\
    &\le \sumT \Tr\left( x_t x_t^\top \left( \sum_{s=1}^{t-1} x_s x_s^\top + \lambda I_{d} \right)^{-1} \right) \,.
\end{align*}
Let us define $M_t := \sum_{s=1}^t x_s x_s^\top + \tfrac{\lambda}{2} I_d $,
we have that $M_t \preccurlyeq \sum_{s=1}^{t-1} x_s x_s^\top + \lambda I_d$ when $\lambda \ge 2$. We obtain
\begin{align*}
    \sumT \lVert \bar W_t^{-1/2} (I_K \otimes x_t) \nabla\mu(\theta_{t+1} x_t)^{1/2} \rVert_2^2 
    &\le \sumT \Tr\left( (M_t - M_{t-1}) M_t^{-1} \right) \\
    &\le \sumT \log \dfrac{|M_t|}{|M_{t-1}|} &&\text{\citep[Lemma~4.5]{hazan2016introduction}} \\
    &\le d \log\left( 1 + \dfrac{T}{\lambda d} \right) &&\text{(Lemma~\ref{lemma:det_trace inequality})} \,.
\end{align*}

\medskip \noindent
\textbf{Step 5: Bounding the sum of \(\lVert \nabla\mu(\theta_* x_t)^{1/2} \rho \rVert_2^2\).}\quad 
We add and subtract a term and get
\begin{align*}
    \sum_{t=1}^T &\lVert \nabla\mu(\theta_* x_t)^{1/2} \rho \rVert_2^2 \\
    &= \sum_{t=1}^T \langle \rho, \nabla\mu(\theta_*x_*)\rho\rangle + \sum_{t=1}^T \langle \rho, (\nabla\mu(\theta_*x_t) - \nabla\mu(\theta_*x_*)) \rho \rangle \\
    &\le R^2 T / \kappa_* + \sum_{t=1}^T \langle \rho, (\nabla\mu(\theta_*x_t) - \nabla\mu(\theta_*x_*)) \rho \rangle \,.
\end{align*}
We use the definition of \(\nabla\mu(.)\) and get
\begin{align}
    \sum_{t=1}^T &\langle \rho, (\nabla\mu(\theta_*x_t) - \nabla\mu(\theta_*x_*)) \rho \rangle \nonumber \\
    &= \sum_{t=1}^T \langle \rho, \diag(\mu(\theta_* x_t) - \mu(\theta_*x_*)) \rho \rangle + \langle \rho, (\mu(\theta_*x_*) \mu(\theta_*x_*)^\top - \mu(\theta_*x_t) \mu(\theta_*x_t)^\top ) \rho \rangle \nonumber \\
    &\le \sumT \sumK \rho_k^2 (\mu(\theta_* x_t)_k - \mu(\theta_* x_*)_k) + \langle \rho, \mu(\theta_* x_*) \rangle^2 - \langle \rho, \mu(\theta_*x_t)\rangle^2 \label{eq:regret bound step 5} \,.
\end{align}
We first tackle the left term. 
%
We add and subtract a term:
\begin{equation*}
    (\rho^{\odot 2})^\top (\mu(\theta_* x_t) - \mu(\theta_* x_*)) = \underbrace{ (\nu \rho + \rho^{\odot 2} )^\top \mu(\theta_* x_t) - (\nu \rho + \rho^{\odot 2} )^\top \mu(\theta_* x_*) }_{\le 0} + \nu \rho^\top (\mu(\theta_* x_*) - \mu(\theta_* x_t)) \,.
\end{equation*}
By definition of $x_*$ the difference is negative, we have:
\begin{equation*}
    (\rho^{\odot 2})^\top (\mu(\theta_* x_t) - \mu(\theta_* x_*)) \le \nu \rho^\top (\mu(\theta_* x_*) - \mu(\theta_* x_t)) \,.
\end{equation*}
Reusing our computation from Step~1, we get
\begin{equation*}
    (\rho^{\odot 2})^\top (\mu(\theta_* x_t) - \mu(\theta_* x_*)) \le 2\nu ( \epsilon_{1,t}(x_t) + \epsilon_{2,t}(x_t) ) \,.
\end{equation*}

We now tackle the right term:
\begin{align*}
    \sumT \langle \rho, \mu(\theta_* x_*) \rangle^2 &- \langle \rho, \mu(\theta_*x_t)\rangle^2 \\
    &= \sumT \langle \rho, \mu(\theta_*x_*) - \mu(\theta_* x_t) \rangle \langle\rho, \mu(\theta_*x_*)+\mu(\theta_*x_t)\rangle \\
    &\le \sumT \langle \rho, \mu(\theta_*x_*) - \mu(\theta_* x_t) \rangle \lVert \rho \rVert_\infty \lVert \mu(\theta_*x_*)+\mu(\theta_*x_t) \rVert_1 &&(\text{CS}) \\
    &\le 2R \sumT \langle \rho, \mu(\theta_*x_*) - \mu(\theta_* x_t) \rangle &&(\lVert \rho \rVert_\infty \le \lVert \rho \rVert_2) \\
    &\le 4R \sumT (\epsilon_{1,t}(x_t) + \epsilon_{2,t}(x_t)) \,. &&(\text{Proposition~\ref{prop:optimistic reward}}) 
\end{align*}
By substituting in Equation~\eqref{eq:regret bound step 5} we have
\begin{equation*}
    \sum_{t=1}^T \langle \rho, (\nabla\mu(\theta_*x_t) - \nabla\mu(\theta_*x_*)) \rho \rangle 
    \le \left( 4R + 2\nu \right) \sumT ( \epsilon_{1,t}(x_t) + \epsilon_{2,t}(x_t) ) \,.
\end{equation*}

\medskip \noindent
\textbf{Step 6: Putting everything together.}
\quad Combining our previous results we get
\begin{align*}
    \Reg_T(\mathrm{Learning}) \le &2 \sum_{t=1}^T \epsilon_{1,t}(x_t) + \epsilon_{2,t}(x_t) \\
    \le &6 R \kappa \sigma_T(\delta)^2 Kd \log\left(1 + \tfrac{T}{Kd\lambda \kappa} \right) \\
    &+e \sigma_T(\delta) \left[ d \log\left( 1 + \tfrac{T}{\lambda d} \right) \right]^{1/2} \left[ \tfrac{R^2 T}{\kappa_*} + \left( 4R + 2\nu \right) \sum_{t=1}^T \epsilon_{1,t}(x_t) + \epsilon_{2,t}(x_t) \right]^{1/2} \\
    \le &6 R \kappa \sigma_T(\delta)^2 Kd \log\left(1 + \tfrac{T}{Kd\lambda \kappa} \right) \\
    &+e \sigma_T(\delta) \left[ d \log\left( 1 + \tfrac{T}{\lambda d} \right) \right]^{1/2} R\sqrt{ \dfrac{T}{\kappa_*} } \\
    &+ e \sigma_T(\delta) \left[ d \log\left( 1 + \tfrac{T}{\lambda d} \right) \right]^{1/2} \left[ \left( 4R + 2\nu \right) \sum_{t=1}^T \epsilon_{1,t}(x_t) + \epsilon_{2,t}(x_t) \right]^{1/2} \,.
\end{align*}
\normalsize
We use the fact that \( x^2 - bx -c \le 0 \implies x^2 \le 2b^2 + 2c \) with \( x^2 =\sum_{t=1}^T \epsilon_{1,t}(x_t) + \epsilon_{2,t}(x_t) \) and get with probability \( 1-2\delta \)
\begin{align*}
    \Reg_T(\mathrm{Learning}) \le &12 R \kappa \sigma_T(\delta)^2 Kd \log\left(1 + \tfrac{T}{Kd\lambda \kappa} \right) \\
    &+ 2e \sigma_T(\delta) \left[ d \log\left( 1 + \tfrac{T}{\lambda d} \right) \right]^{1/2} R \sqrt{ \dfrac{ T}{\kappa_*} } \\
    &+4 e^2 \sigma_T(\delta)^2 (2R+\nu)  d \log\left( 1 + \tfrac{T}{\lambda d} \right) \\
    \le &\cste d \sqrt{K} \log(T/\delta) R \sqrt{\dfrac{T}{\kappa_*}} + \cste (1+S) \left( R + \nu \right) \kappa K^2 d^{2} \log^2(T/\delta) \,.
\end{align*}
where applying the Union Bound gives the result with probability \(1-2\delta\).

\end{proof}

\subsection{The Action \texorpdfstring{$x_*$}{x*} is Well-defined}\label{appendix:x_* well defined}
In this section, we establish the existence of the optimal action $x_*$ used in the proof of Theorem~\ref{thm:regret bound}. 

\begin{restatable}{lmm}{}\label{lemma:x_* exists}
Let $\cX_* := \argmax_{x\in\cX} \rho^\top \mu(\theta_* x) $ and $x_* \in \cX_*$. Let us define $\nu$ as 
\begin{equation*}
    \nu := \dfrac{2 \max_{x_1, x_2 \in\cX} | (\rho^{\odot 2})^\top (\mu(\theta_* x_1) - \mu(\theta_* x_2) | }{ \rho^\top \mu(\theta_* x_*) - \max_{x\in\cX \backslash \cX_*} \rho^\top \mu(\theta_* x) } \,.
\end{equation*}
Then
\begin{equation*}
    \left( \argmax_{x\in\cX} (\nu \rho + \rho^{\odot 2})^\top \mu(\theta_* x) \right) \subseteq \left( \argmax_{x\in\cX} \rho^\top \mu(\theta_* x) \right) \,.
\end{equation*}
\end{restatable}

\begin{proof}


Let $ x_*^\nu \in \argmax_{x\in \cX} \nu \rho^\top \mu(\theta_* x) + (\rho^{\odot2})^\top \mu(\theta_* x) $. 
We show by contradiction that $x_*^\nu \in \cX_*$. Let us assume that $x_*^\nu \notin \cX_*$. By the definition of $x_*^\nu$ we have
\begin{equation*}
    \nu \rho^\top \mu(\theta_* x_*) + (\rho^{\odot 2})^\top \mu(\theta_* x_*) \le \nu \rho^\top \mu(\theta_* x_*^\nu) + (\rho^{\odot 2})^\top \mu(\theta_* x_*^\nu) \,.
\end{equation*}
We rearrange the terms to get
\begin{equation*}
    \nu \rho^\top ( \mu(\theta_* x_*) - \mu(\theta_* x_*^\nu))  \le (\rho^{\odot 2})^\top ( \mu(\theta_* x_*^\nu) - \mu(\theta_* x_*) ).
\end{equation*}
Using the definition of $\nu$ we get
\begin{equation*}
    2 \max_{x_1, x_2 \in\cX} | (\rho^{\odot 2})^\top (\mu(\theta_* x_1) - \mu(\theta_* x_2) | 
    \le (\rho^{\odot 2})^\top ( \mu(\theta_* x_*^\nu) - \mu(\theta_* x_*) ).
\end{equation*}
We obtain the contradiction $ 2 \le 1 $ and thus $x_*^\nu \in \cX_* $.


\end{proof}

\newpage
\section{Proof of Theorem~\ref{thm:lower bound}}\label{appendix:proof lower bound}
\ThmLowerBound*

\begin{proof}
We use the canonical bandit probability space \( (\Omega_t, \F_t, \pr_{\pi \theta \rho}) \) of \citet[Section~4.7]{lattimore2020bandit}. To simplify let us denote \( \pr_\theta = \pr_{\pi \theta \rho} \) the probability of the random sequence \( \{x_1, y_1, \dots , x_T, y_T \} \) obtained by having the algorithm \(\pi\) interact with the environment \((\theta, \rho)\). The expectation \(\E_\theta\) is computed with respect to the probability \(\pr_\theta\).

\medskip
We start by defining an instance of a MNL bandit problem.
Let \(\theta_0 = \Pi M_0\) with \( M_0\in\R^{K\times d} \) be defined as follows
\begin{equation*}
M_0 := \dfrac{1}{\sqrt{K+3}}
\begin{bmatrix}
    2 & 0 &\dots & 0 \\
    1 & 0 & \dots & 0 \\
    \vdots & \vdots & \ddots & \vdots \\
    1 & 0 & \dots & 0
\end{bmatrix} 
\end{equation*}
and \( \rho\in\R^K \) be defined by 
\begin{equation*}
\rho := \dfrac{R}{\sqrt{K+3}}
\begin{bmatrix}
    2 \\
    1 \\
    \vdots \\
    1
\end{bmatrix} \,.
\end{equation*}
Even though this defines a binary problem, we cannot directly apply the proof of \citet{abeille2021instance} as for \(K\ge 3\) our \(\kappa_*\) will be different than in the binary setting. Indeed the probability distributions of the reward are different, see the \(\ln(K-1)\) term in Equation~\eqref{eq:proof lower bound eq proba}.

We define the action set by the sphere \( \cX = \mathcal{S}_1(\R^d) \). 
We show that a slight variation \(\tilde M\) of the matrix \(M_0\) results in a regret lower-bounded by \(  \Omega(Rd\sqrt{KT/\kappa_*(\theta)} )\). 
Let us define the set of perturbed matrices \( \cM \) by
\begin{equation*}
    \cM := \left\{ M_0 + \epsilon \sum_{i=2}^d v_i e_{1i} + \dfrac{\epsilon}{2} \sum_{k=2}^K \sum_{i=2}^d v_i e_{ki}
    \quad , v\in\{-1, 1\}^d \right\} 
\end{equation*}
where \( \epsilon >0\) is to be defined later. 
%
For now we only assume that \[ {\epsilon \le \lVert [M_0]_1 \rVert_2 / \sqrt{d-1} = 2 / \sqrt{(K+3)(d-1)} } \,.\] 
Note that we do not modify the first column. 
Let \( \smash{x_*(\theta) := \argmax_{x\in\cX} \rho^\top \mu(\theta x) } \). Let $M\in\R^{K\times d}$, as for \(\theta = \Pi M\) we have \( \mu(\theta x) = \mu(Mx) \), we may abuse the notation and write \( x_*(M) = x_*(\theta) \). For every problem instance \( \smash{(\theta = \Pi M \in \Pi\cM, \rho, \mathcal{S}_1(\R^d) ) } \), we have that \( \smash{x_*(\theta) = [M]_1 / \lVert [M]_1 \rVert_2 }\).

\medskip
We introduce a second set \(\smash{\tilde \cM \subseteq \R^{K\times d}}\) of matrices, 
which is in bijection with \(\cM\). This alternative set simplifies the presentation 
of the proof, but should be regarded as equivalent to \(\cM\). It is defined as follows:
\begin{equation*}
    \tilde \cM := \left\{ 
    \begin{bmatrix}
        & [M]_1 / \lVert [M]_1 \rVert_2 & \\
        0 & \dots & 0 \\
        \vdots & \ddots & \vdots \\
        0 & \dots & 0
    \end{bmatrix} 
    : M \in \cM 
    \right\} \,.
\end{equation*}
We denote by $ \gamma : \cM \to \tilde \cM $ the canonical bijection from $\cM$ to $\tilde \cM$. For all $M\in\cM$ we have that $ \argmax_{x\in\cX} \rho^\top \mu(M x) = \argmax_{x\in\cX} \rho^\top \mu(\gamma(M)x) $. 

\medskip
We assume that for all \(\tilde M \in \tilde\cM\),
\(
\Reg_T(\Pi \tilde M) \leq R d \sqrt{KT/\kappa_*},
\)
which can be done without loss of generality, since otherwise the lower-bound 
already holds.
The proof then consists in showing that there exists a matrix $M_*\in \tilde \cM$ such that, for $\theta_* = \Pi M_*$, the regret is lower-bounded as $\Reg_T(\theta_*) \ge \cste Rd\sqrt{KT/\kappa_*}$.

\medskip \noindent
\textbf{Step 1: Lower-bounding by the optimum regime.}
In this step, we follow the idea of Proposition~6 from \citet{abeille2021instance}. Let \(\ttheta = \Pi \tilde M \in \Pi\tilde\cM\), we lower-bound the regret $\smash{\Reg_T(\ttheta)}$ by the derivative of the sigmoid function in the optimum $\smash{\mu'\left([\tilde M]_1^\top x_*(\ttheta) -\ln(K-1)\right)}$.
We first express the regret \(\Reg_T(\ttheta)\) in terms of Bernoulli variables.
\begin{align*}
    \Reg_T(\ttheta) &:= \sumT \rho^\top \mu(\ttheta x_*(\ttheta)) - \rho^\top \mu(\ttheta x_t) \\
    &= \sumT \rho^\top \mu(\tilde M x_*(\theta)) - \rho^\top \mu(\tilde M x_t) \\
    &=\sumT \rho_1 [ \pr_{\ttheta}(\rho_1 | x_*(\ttheta)) - \pr_{\ttheta}(\rho_1 | x_t) ] + \rho_2 [ \pr_{\ttheta}(\rho_2 | x_*(\ttheta)) - \pr_{\ttheta}(\rho_2 | x_t) ] \\
    &= \sumT \rho_1 [ \pr_{\ttheta}(\rho_1 | x_*(\ttheta)) - \pr_{\ttheta}(\rho_1 | x_t) ] + \rho_2 [ 1 - \pr_{\ttheta}(\rho_1 | x_*(\ttheta)) - 1 + \pr_{\ttheta}(\rho_1 | x_t) ]
\end{align*}
where we have $ \pr_{\ttheta}(\rho_1 | x) = [\mu(\tilde M x)]_1 $ and $ \pr_{\ttheta}(\rho_2 | x) = \sum_{k=2}^K [\mu(\tilde M x)]_k $.
Using the definition of $\rho$ we get
\begin{equation*}
    \Reg_T(\ttheta) = \dfrac{R}{\sqrt{K+3}} \sumT [ \pr_{\ttheta}(\rho_1 | x_*(\ttheta)) - \pr_{\ttheta}(\rho_1 | x_t) ] \,.
\end{equation*}
Substituting,
\begin{multline}\label{eq:proof lower bound eq proba}
    \pr_{\ttheta}(\rho_1 | x) = [\mu(\tilde Mx)]_1 = \dfrac{\exp([\tilde M]_1^\top x)}{\exp([\tilde M]_1^\top x) + (K-1) } \\
    = \dfrac{1}{1+\exp(-[\tilde M]_1^\top x + \ln(K-1))}
    = \mu([\tilde M]_1\top x - \ln(K-1)) \,.
\end{multline}
Thus we get
\begin{equation*}
    \Reg_T(\ttheta) \overset{\text{\eqref{eq:proof lower bound eq proba}}}{=} \dfrac{R}{\sqrt{K+3}} \sumT \mu \left( [\tilde M]_1^\top x_*(\ttheta) -\ln(K-1) \right) - \mu \left( [\tilde M]_1^\top x_t -\ln(K-1) \right) \,.
\end{equation*}
We now apply the Mean-value Theorem:
\begin{equation}
    \Reg_T(\ttheta) = \dfrac{R}{\sqrt{K+3}} \sumT \int_0^1 \mu'\left( [\tilde M]_1^\top (v x_*(\ttheta) + (1-v) x_t) - \ln(K-1) \right) \mathrm{d}v \left([\tilde M]_1^\top (x_*(\ttheta) - x_t) \right) \,. \label{eq:proof lower bound eq 2} 
\end{equation}
Using the self-concordance property \citep[Corollary~2]{sun2019generalized} on $\mu'$ between $[\tilde M]_1^\top x_t$ and $[\tilde M]_1^\top x_*(\ttheta)$, we get
\begin{align*}
    \Reg_T(\ttheta) &\ge \dfrac{R}{\sqrt{K+3}} \dfrac{1}{1+ \llVert [\tilde M]_1^\top (x_*(\ttheta) - x_t) \rrVert_2 } \sumT \mu'\left( [\tilde M]_1^\top x_*(\ttheta) -\ln(K-1) \right) \left( [\tilde M]_1^\top (x_*(\ttheta) - x_t) \right) \\
    &\ge \dfrac{R}{\sqrt{K+3}} \dfrac{1}{1+ \llVert [\tilde M]_1 \rrVert_2 \lVert (x_*(\ttheta) - x_t) \rVert_2 } \sumT \mu'\left([\tilde M]_1^\top x_*(\ttheta) -\ln(K-1)\right) \left( [\tilde M]_1^\top (x_*(\ttheta) - x_t) \right) \\
    &\ge \dfrac{R}{\sqrt{K+3}} \dfrac{1}{1+ 2\lVert [\tilde M]_1 \rVert_2 } \sumT \mu'\left([\tilde M]_1^\top x_*(\ttheta) -\ln(K-1)\right) \left([\tilde M]_1^\top (x_*(\ttheta) - x_t) \right) \\
    &\ge \dfrac{R}{3\sqrt{K+3}} \sumT \mu'\left( [\tilde M]_1^\top x_*(\ttheta) -\ln(K-1)\right) \left( [\tilde M]_1^\top (x_*(\ttheta) - x_t) \right)
\end{align*}
where the second inequality is by Cauchy-Schwarz inequality, the third inequality is because the actions are in the sphere $\mathcal{S}_1(\R^d) $, and the last inequality is because $\lVert [\tilde M]_1 \rVert_2 = 1 $. Using the definition of \(x_*\) we have that
\begin{equation*}
    ([\tilde M]_1^\top x_*(\ttheta) - [\tilde M]_1^\top x_t) = \lVert [\tilde M]_1 \rVert_2 \left( 1 - \tfrac{[\tilde M]_1^\top}{\lVert [\tilde M]_1 \rVert_2} x_t \right) =  \lVert [\tilde M]_1 \rVert_2 \tfrac{1}{2} \lVert x_*(\ttheta) - x_t \rVert_2^2 
\end{equation*}
where the last equality is due to \( 1 - x^\top y = \tfrac{1}{2} \lVert x - y \rVert_2^2 \) for all \( x, y \in \mathcal{S}_1(\R^d) \). Thus we obtain
\begin{align}
    \Reg_T(\ttheta) &\ge \dfrac{R\lVert [\tilde M]_1 \rVert_2}{6\sqrt{K+3}} \sumT \mu'\left([\tilde M]_1^\top x_*(\ttheta) -\ln(K-1)\right) \lVert x_*(\ttheta) - x_t \rVert_2^2 \nonumber \\
    &=\dfrac{R}{6\sqrt{K+3}} \sumT \mu'\left([\tilde M]_1^\top x_*(\ttheta) -\ln(K-1)\right) \lVert x_*(\ttheta) - x_t \rVert_2^2 \label{eq:proof lower bound eq 1} \\
    &\ge \dfrac{R}{6\sqrt{K+3}} \mu'\left([\tilde M]_1^\top x_*(\ttheta) -\ln(K-1)\right) \sumT \sum_{i=2}^d [ x_*(\ttheta) - x_t ]^2_i  \,. \nonumber
\end{align}

Let us denote $ M = \gamma^{-1}(\tilde M) \in \cM $ and $\theta = \Pi M$, we have $x_*(\theta) = x_*(\ttheta)$. Thus 
\begin{equation}\label{eq:inside proof lower bound 3}
    \Reg_T(\tilde \theta) \ge \dfrac{R}{6\sqrt{K+3}} \mu'\left([\tilde M]_1^\top x_*(\ttheta) -\ln(K-1)\right) \sumT \sum_{i=2}^d [ x_*(\theta) - x_t ]^2_i  \,.
\end{equation}
Let us define the event $A_i(\theta)$ for all $i\in \llbracket d \rrbracket$ and all $\theta \in \Pi \R^{K\times d}$ as 
\begin{equation*}
    A_i(\theta) := \left\{ [x_*(\theta) - x_*(\theta_0)]_i \cdot \left[ x_*(\theta_0) - \dfrac{1}{T} \sumT x_t \right] \ge 0 \right\} \,.
\end{equation*}
We bound the regret of any $ \ttheta = \Pi \tilde M \in \Pi \tilde \cM $ using the event $A_i(\theta)$.
By applying Lemma~3 of \citep{abeille2021instance} we obtain
\begin{equation*}
    \sumT \sum_{i=2}^d [ x_*(\theta) - x_t ]^2_i \ge \dfrac{3T\epsilon^2}{8\lVert [M]_1 \rVert_2^2} \sum_{i=2}^d \pr_{\ttheta} (A_i(\theta)) \,.
\end{equation*}
We apply this result in Equation~\eqref{eq:inside proof lower bound 3} to get
\begin{align}
    \E_{\ttheta}\left[\Reg_T(\ttheta)\right] &\ge \dfrac{RT\epsilon^2}{16\sqrt{K+3}\lVert [M]_1 \rVert_2^2} \mu'\left([\tilde M]_1^\top x_*(\ttheta) -\ln(K-1)\right) \sum_{i=2}^d \pr_{\ttheta} (A_i(\theta)) \nonumber \\
    &= \dfrac{RT\epsilon^2\sqrt{K+3}}{64} \mu'\left([\tilde M]_1^\top x_*(\ttheta) -\ln(K-1)\right) \sum_{i=2}^d \pr_{\ttheta} (A_i(\theta)) \label{eq:inside proof lower bound 7} \,.
\end{align}


\medskip \noindent
\textbf{Step 2: Showing that $\mu'\big([\tilde M]_1^\top x_*(\ttheta) -\ln(K-1)\big) = (K+3) / \kappa_*(\ttheta)$.}\\
Recall that $ \kappa_*(\ttheta) $ is defined by
\begin{equation*}
    \kappa_*(\ttheta)^{-1} := \dfrac{\rho^\top \nabla\mu(\ttheta x_*(\ttheta)) \rho }{\lVert \rho \rVert_2^2} = \dfrac{\rho^\top \mathrm{diag}(\mu(\ttheta x_*(\ttheta))) \rho - \rho^\top \mu(\ttheta x_*(\ttheta)) \mu(\ttheta x_*(\ttheta))^\top \rho}{R^2} \,.
\end{equation*}
This develops into
\begin{align*}
    \dfrac{R^2}{\kappa_*(\ttheta)} &= \rho_1^2 \mu(\ttheta x_*(\ttheta))_1 ( 1 - \mu(\ttheta x_*(\ttheta))_1 
    + \sum_{k=2}^K \rho_k \left[ \mu(\ttheta x_*(\ttheta))_k \sum_{i=2}^K \rho_i \left( \delta_{ik} - \mu(\ttheta x_*(\ttheta))_i \right) \right] \\
    &\qquad - 2  \rho_1 \mu(\ttheta x_*(\ttheta))_1 \sum_{k=2}^K \rho_k \mu(\ttheta x_* (\ttheta)_k  \,.
\end{align*}

We start by using the definition of $\mu'$:
\begin{align*}
    \mu'([\tilde M]_1^\top x_*(\ttheta) -\ln(K-1)) 
    &= \pr_{\ttheta}(\rho_2 | x_*(\ttheta)) (1-\pr_{\ttheta}(\rho_2 | x_*(\ttheta)) ) \\
    &= \sum_{k=2}^K \mu(\ttheta x_*(\ttheta))_k \left(1 - \sum_{i=2}^K \mu(\ttheta x_*(\ttheta))_i \right) \\
    &= \sum_{k=2}^K \left[ \mu(\ttheta x_*(\ttheta))_k \sum_{i=2}^K \delta_{ik} - \mu(\ttheta x_*(\ttheta))_i \right] \\
    &= \sum_{k=2}^K \dfrac{\rho_k}{\rho_k} \left[ \mu(\ttheta x_*(\ttheta))_k \sum_{i=2}^K \dfrac{\rho_i}{\rho_i} \left( \delta_{ik} - \mu(\ttheta x_*(\ttheta))_i \right) \right] \,.
\end{align*}
Using the definition of $\rho$ we get 
\begin{equation}\label{eq:inside proof lower bound 4}
    \mu'([\tilde M]_1^\top x_*(\ttheta) -\ln(K-1))
    = \dfrac{K+3}{R^2} \sum_{k=2}^K \rho_k \left[ \mu(\ttheta x_*(\ttheta))_k \sum_{i=2}^K \rho_i \left( \delta_{ik} - \mu(\ttheta x_*(\ttheta))_i \right) \right] \,.
\end{equation}
Let us now consider $ 4 \mu(\ttheta x_*(\ttheta))_1 (1-\mu(\ttheta x_*(\ttheta))_1) $:
\begin{align*}
    4 \mu(\ttheta x_*(\ttheta))_1 (1-\mu(\ttheta x_*(\ttheta))_1) &= 
    4 \dfrac{\rho_1^2}{\rho_1^2} \mu(\ttheta x_*(\ttheta))_1 (1-\mu(\ttheta x_*(\ttheta))_1) \\
    &= 4 \dfrac{K+3}{4R^2} \rho_1^2 \mu(\ttheta x_*(\ttheta))_1 (1-\mu(\ttheta x_*(\ttheta))_1) \\
    &= \dfrac{K+3}{R^2} \rho_1^2 \mu(\ttheta x_*(\ttheta))_1 (1-\mu(\ttheta x_*(\ttheta))_1) \,.
\end{align*}
We can write $ 4 \mu(\ttheta x_*(\ttheta))_1 (1-\mu(\ttheta x_*(\ttheta))_1) $ differently to obtain:
\begin{align*}
    4 \mu(\ttheta x_*(\ttheta))_1 (1-\mu(\ttheta x_*(\ttheta))_1) &= 
    4 \mu(\ttheta x_*(\ttheta))_1 \sum_{k=2}^K \mu(\ttheta x_*(\ttheta))_k \\
    &= 4 \dfrac{\rho_1}{\rho_1} \mu(\ttheta x_*(\ttheta))_1 \sum_{k=2}^K \dfrac{\rho_k}{\rho_k} \mu(\ttheta x_*(\ttheta))_k \\
    &= 2 \cdot 2 \dfrac{K+3}{2R^2} \rho_1 \mu(\ttheta x_*(\ttheta))_1 \sum_{k=2}^K \rho_k \mu(\ttheta x_*(\ttheta))_k  \\
    &= 2 \dfrac{K+3}{R^2} \rho_1 \mu(\ttheta x_*(\ttheta))_1 \sum_{k=2}^K \rho_k \mu(\ttheta x_*(\ttheta))_k \,.
\end{align*}
We add and subtract $4\mu(\ttheta x_*(\ttheta))_1 (1-\mu(\ttheta x_*(\ttheta))_1)$ in Equation~\eqref{eq:inside proof lower bound 4} to obtain the desired result:
\begin{align*}
    \mu'([\tilde M]_1^\top &x_*(\ttheta) -\ln(K-1)) \\
    &= \dfrac{K+3}{R^2} \sum_{k=2}^K \rho_k \left[ \mu(\ttheta x_*(\ttheta))_k \sum_{i=2}^K \rho_i \left( \delta_{ik} - \mu(\ttheta x_*(\ttheta))_i \right) \right] \\
    &\qquad +\dfrac{K+3}{R^2} \rho_1^2 \mu(\ttheta x_*(\ttheta))_1 (1-\mu(\ttheta x_*(\ttheta))_1) -2 \dfrac{K+3}{R^2} \rho_1 \mu(\ttheta x_*(\ttheta))_1 \sum_{k=2}^K \rho_k \mu(\ttheta x_*(\ttheta))_k \\
    &= \dfrac{K+3}{\kappa_*(\ttheta)}\,.
\end{align*}
By substituting into Equation~\eqref{eq:inside proof lower bound 7} we obtain
\begin{equation}\label{eq:inside proof lower bound 5}
    \Reg_T(\ttheta) \ge \dfrac{RT\epsilon^2(K+3)^{3/2}}{64\kappa_*(\ttheta)} \sum_{i=2}^d \pr_{\ttheta} (A_i(\theta)) \,.
\end{equation}

\medskip\noindent
\textbf{Step 3: Averaging Hammer and Average Relative Entropy.} \quad
Let us define $ \smash{\Xi := \Pi \tilde \cM } $. In order to find a $\ttheta \in \Xi$ with a large regret lower-bound, we use the averaging hammer technique as in \citet[Section~24.1]{lattimore2020bandit}.
Let us recall Lemma~4 from \citep{abeille2021instance}, the following holds:
\begin{equation}\label{eq:inside proof lower bound 8}
    \dfrac{1}{|\Xi |} \sum_{\ttheta \in \Xi} \sum_{i=2}^d \pr_{\ttheta} (A_i(\theta)) \ge \dfrac{d}{4} - \dfrac{\sqrt{d}}{2} \sqrt{ \dfrac{1}{|\Xi |} \sum_{\ttheta \in \Xi} \sum_{i=2}^d KL \left(\pr_{\ttheta}, \pr_{\flipi(\ttheta)} \right) } \,.
\end{equation}
where the flipping operator $\flipi$ is defined by
\begin{equation*}
    [\flipi(\theta)]_i = -[\theta]_i \qquad\text{and}\qquad [\flipi(\theta)]_j = [\theta]_j \text{ for all } j\neq i \,.
\end{equation*}

We study the average relative entropy and upper-bound it by the regret.
Let us denote \( P_{x_t}^{\ttheta} = \pr_{\ttheta}(\cdot|x) \). Using the Divergence Decomposition Lemma \citep[Exercise~15.8(b)]{lattimore2020bandit} and the fact that the \(\chi^2\)-divergence upper-bounds the KL divergence we get
\begin{equation}\label{eq:inside proof lower bound 1}
    KL \left( \pr_{\ttheta}, \pr_{\flipi(\ttheta)} \right) = \E_{\ttheta} \left[ \sum_{t=1}^T KL\left( P_{x_t}^{\ttheta}, P_{x_t}^{\flipi(\ttheta)} \right) \right] \le \E_{\ttheta} \left[ \sum_{t=1}^T D_{\chi^2} \left(  P_{x_t}^{\ttheta}, P_{x_t}^{\flipi(\ttheta)} \right) \right] \,.
\end{equation}
Remember that we have
\begin{equation*}
    \pr_{\ttheta}(\rho_1 | x) = \dfrac{1}{1+\exp(-[\tilde M]_1^\top x + \ln(K-1))} = \mu\left( [\tilde M]_1^\top x - \ln(K-1) \right) \,.
\end{equation*}
Thus the multinomial variables $P_{x_t}^{\ttheta}$ and $P_{x_t}^{\flipi(\ttheta)}$ can be written as Bernoulli variables.
Therefore by substituting in Equation~\eqref{eq:inside proof lower bound 1} we have that
\small
\begin{multline*}
    KL \left( \pr_{\ttheta}, \pr_{\flipi(\ttheta)} \right) \\
    \le \E_{\ttheta} \left[ \sum_{t=1}^T  D_{\chi^2} \left( \mathrm{Bernoulli}\left(\mu\left([\tilde M]_1^\top  x_t - \ln(K-1)\right)\right), \mathrm{Bernoulli}\left(\mu\left( [\flipi( \tilde M)]_1^\top  x_t - \ln(K-1)\right)\right) \right) \right] \,.
\end{multline*}
\normalsize
Using the expression of the \(\chi^2\)-divergence for Bernoulli random variables gives
\begin{equation*}
    KL \left( \pr_{\ttheta}, \pr_{\flipi(\ttheta)} \right) \le \E_\theta \left[ \sum_{t=1}^T \dfrac{ \left( \mu\left( [\tilde M]_1^\top  x_t - \ln(K-1)\right) - \mu\left( [\flipi( \tilde M)]_1^\top  x_t - \ln(K-1)\right) \right)^2 }{\mu'\left( [\flipi(\tilde M)]_1^\top  x_t - \ln(K-1)\right)} \right] \,.
\end{equation*}
We apply the Mean-value Theorem and get 
\begin{equation*}
    KL \left( \pr_{\ttheta}, \pr_{\flipi(\ttheta)} \right)
    \le \E_{\ttheta} \left[ \sum_{t=1}^T  \tfrac{\left( \int_0^1 \mu'\left( \left( v [\tilde M]_1 + (1-v) [\flipi(\tilde M)]_1 \right)^\top x_t - \ln(K-1) \right)\mathrm{d}v\right)^2 }{\mu'\left([\flipi(\tilde M)]_1^\top  x_t - \ln(K-1)\right)} \left(( [\tilde M]_1 - [\flipi(\tilde M)]_1 )^\top x_t \right)^2 \right] \,.
\end{equation*}
Now applying the self-concordance property gives
\begin{align*}
    &\int_0^1 \mu'\left( \left( v [\tilde M]_1 + (1-v) [\flipi(\tilde M)]_1 \right)^\top x_t - \ln(K-1) \right)\mathrm{d}v \\
    &\qquad \le \mu'\left([\flipi(\tilde M)]_1^\top  x_t - \ln(K-1)\right) \exp\left( \left| ([\tilde M]_1 - [\flipi(\tilde M)]_1)^\top  x_t \right| \right) \\
    \text{and}& \\
    &\int_0^1 \mu'\left( \left( v [\tilde M]_1 + (1-v) [\flipi(\tilde M)]_1 \right)^\top x_t - \ln(K-1) \right)\mathrm{d}v \\
    &\qquad \le \mu'\left([\tilde M]_1^\top  x_t - \ln(K-1)\right) \exp\left( \left| ([\tilde M]_1 - [\flipi(\tilde M)]_1)^\top  x_t \right| \right) \,.
\end{align*}
Thus we obtain
\footnotesize
\begin{align*}
    KL &\left( \pr_{\ttheta}, \pr_{\flipi(\ttheta)} \right) \\
    &\le \E_{\ttheta} \left[ \sum_{t=1}^T \mu'\left([\tilde M]_1^\top  x_t /2 - \ln(K-1)\right) \left( ( [\tilde M]_1 - [\flipi(\tilde M)]_1 )^\top x_t \right)^2 \exp\left( 2\left| ([\tilde M]_1 - [\flipi(\tilde M)]_1)^\top  x_t \right| \right)  \right] \\
    &\le \exp(2\epsilon) \E_{\ttheta} \left[ \sum_{t=1}^T \mu'\left([\tilde M]_1^\top  x_t - \ln(K-1)\right) \left( ( [\tilde M]_1 - [\flipi(\tilde M)]_1 )^\top x_t \right)^2 \right] \\
    &\le \exp(2\epsilon) 4 \epsilon^2 \E_{\ttheta} \left[ \sum_{t=1}^T \mu'\left([\tilde M]_1^\top  x_t - \ln(K-1)\right) [x_t]_i^2 \right] \,.
\end{align*}
\normalsize
We add and subtract $x_*(\ttheta)$ and apply Young's Inequality:
\begin{align*}
    KL &\left( \pr_{\ttheta}, \pr_{\flipi(\ttheta)} \right) \\
    &\le \exp(2\epsilon) 4\epsilon^2 \E_{\ttheta} \left[ \sum_{t=1}^T \mu'\left([\tilde M]_1^\top  x_t - \ln(K-1)\right) [x_*(\ttheta)-x_t - x_*(\ttheta)]_i^2 \right]\\
    &\le 8 \exp(2\epsilon) \epsilon^2 \E_{\ttheta} \left[ \sum_{t=1}^T \mu'\left([\tilde M]_1^\top  x_t - \ln(K-1)\right) [x_*(\ttheta)-x_t]_i^2 + \sum_{t=1}^T \mu'\left([\tilde M]_1^\top  x_t - \ln(K-1)\right) [x_*(\ttheta)]_i^2 \right] \,.
\end{align*}
Thus by summing over \(d\) we obtain
\footnotesize
\begin{align}
    \sum_{i=2}^d &KL \left( \pr_{\ttheta}, \pr_{\flipi(\ttheta)} \right) \nonumber \\
    &\le 8\exp(2\epsilon) \epsilon^2 \E_{\ttheta} \left[ \sum_{t=1}^T \sum_{i=2}^d \mu'\left([\tilde M]_1^\top  x_t - \ln(K-1)\right) [x_t - x_*(\ttheta)]_i^2 + \sum_{t=1}^T \sum_{i=2}^d \mu' \left([\tilde M]_1^\top  x_t - \ln(K-1)\right) [x_*(\ttheta)]_i^2 \right] \nonumber \\
    &\le 8\exp(2\epsilon) \epsilon^2 \E_{\ttheta} \left[ \sum_{t=1}^T \sum_{i=1}^d \mu' \left([\tilde M]_1^\top  x_t - \ln(K-1)\right) [x_t - x_*(\ttheta)]_i^2 + \sum_{t=1}^T \sum_{i=2}^d \mu' \left([\tilde M]_1^\top  x_t - \ln(K-1)\right) [x_*(\ttheta)]_i^2 \right] \nonumber \\
    &= 8\exp(2\epsilon) \epsilon^2 \E_{\ttheta} \left[ \sum_{t=1}^T \mu' \left([\tilde M]_1^\top  x_t - \ln(K-1)\right) \lVert x_t - x_*(\ttheta)\rVert_2^2 + \sum_{t=1}^T \sum_{i=2}^d \mu' \left([\tilde M]_1^\top  x_t - \ln(K-1)\right) [x_*(\ttheta)]_i^2 \right] \nonumber \\
    &\le 8\exp(2\epsilon) \epsilon^2 \E_{\ttheta} \left[ \sum_{t=1}^T \mu' \left([\tilde M]_1^\top  x_t - \ln(K-1)\right) \lVert x_t - x_*(\ttheta)\rVert_2^2 + \tfrac{(d-1)\epsilon^2}{\lVert [M_0]_1 \rVert_2^2 +(d-1)\epsilon^2} \sum_{t=1}^T \mu' \left([\tilde M]_1^\top  x_t - \ln(K-1)\right) \right] \nonumber \\
    &\le 8\exp(2\epsilon) \epsilon^2 \E_{\ttheta} \left[ \sum_{t=1}^T \mu' \left([\tilde M]_1^\top  x_t - \ln(K-1)\right) \lVert x_t - x_*(\ttheta)\rVert_2^2 + \tfrac{K+3}{4}(d-1) \epsilon^2 \sum_{t=1}^T \mu' \left([\tilde M]_1^\top  x_t - \ln(K-1)\right) \right] \label{eq:inside proof lower bound 6}
\end{align}
\normalsize

where for the second to last inequality we use the fact that \( \lVert[M_0]_1 \rVert_2 = 2/\sqrt{K+3} \). 

\medskip\noindent
\textbf{Step 4: Bounding the First Term of the Average Entropy.}\\
In this step, we upper-bound $\smash{\sumT \mu' \left([\tilde M]_1^\top  x_t - \ln(K-1)\right) \lVert x_t - x_*(\ttheta)\rVert_2^2 }$ using $\smash{\Reg_T(\ttheta)}$.
We follow the Step~1 up to Equation~\eqref{eq:proof lower bound eq 2} and get
\begin{equation*}
    \Reg_T(\ttheta) \ge \dfrac{R}{\sqrt{K+3}} \sumT \int_0^1 \mu'\left( [\tilde M]_1^\top (v x_*(\ttheta) + (1-v) x_t) - \ln(K-1) \right) \mathrm{d}v \left( [\tilde M]_1^\top (x_*(\ttheta) - x_t) \right) \,.
\end{equation*}
We apply the self-concordance property and obtain
\begin{equation*}
    \Reg_T(\ttheta) \ge \dfrac{R}{\sqrt{K+3}} \dfrac{1}{1+ \llVert [\tilde M]_1^\top (x_*(\ttheta) - x_t) \rrVert_2} \sumT \mu' \left( [\tilde M]_1^\top x_t - \ln(K-1) \right) \left( [\tilde M]_1^\top (x_*(\ttheta) - x_t) \right) \,.
\end{equation*}
We now follow our previous computations between Equation~\eqref{eq:proof lower bound eq 2} and~\eqref{eq:proof lower bound eq 1} to get
\begin{equation*}
    \sumT \mu'\left( [\tilde M]_1^\top x_t -\ln(K-1) \right) \lVert x_t - x_*(\ttheta) \rVert_2^2 \le 6 \dfrac{\sqrt{K+3}}{R} \Reg_T(\ttheta) \,.
\end{equation*}

\medskip\noindent
\textbf{Step 5: Bounding the Second Term of the Average Entropy.}
In this step, we upper-bound $\smash{\sumT \mu' \left([\tilde M]_1^\top  x_t - \ln(K-1)\right)}$ using $\smash{\Reg_T(\ttheta)}$ and $\smash{\kappa_*(\ttheta)}$.
We apply a Taylor decomposition with integral remainder:
\begin{align*}
    \sum_{t=1}^T &\mu' \left([\tilde M]_1^\top  x_t - \ln(K-1)\right) \\
    &=\sum_{t=1}^T \mu' \left([\tilde M]_1^\top  x_*(\ttheta) - \ln(K-1)\right) \\
    &\qquad+ \int_0^1 \mu''\left( [\tilde M]_1^\top  x_*(\ttheta) + v[\tilde M]^\top_1(x_t-x_*(\ttheta)) -\ln(K-1) \right) \mathrm{d}v \left( [\tilde M]_1^\top (x_t - x_*(\ttheta)) \right) \\
    &\le \sum_{t=1}^T  \mu' \left([\tilde M]_1^\top  x_*(\ttheta) - \ln(K-1)\right) \\
    &\qquad + \left| \int_0^1 \mu''\left( [\tilde M]_1^\top  x_*(\ttheta) + v[\tilde M]^\top_1 (x_t-x_*(\ttheta)) -\ln(K-1) \right) \mathrm{d}v \right| \cdot \left| [\tilde M]_1^\top  (x_t - x_*(\ttheta)) \right| \,.
\end{align*}
Using the facts that for the sigmoid \( | \mu''| \le \mu' \) and \( [\tilde M]_1^\top  x_*(\ttheta) \ge [\tilde M]_1^\top  x_t \) we get
\begin{align*}
    \sum_{t=1}^T \mu' &\left([\tilde M]_1^\top  x_t - \ln(K-1)\right) \\
    &\le \sum_{t=1}^T \mu' \left([\tilde M]_1^\top  x_*(\ttheta) - \ln(K-1)\right) \\
    &\qquad+ \int_0^1 \mu'\left( [\tilde M]_1^\top (v x_*(\ttheta) + (1-v) x_t) - \ln(K-1) \right) \mathrm{d}v \cdot \left( [\tilde M]_1^\top (x_*(\ttheta) - x_t) \right) \,.
\end{align*}
The first term of the sum can be rewritten using Step~3:
\begin{equation*}
    \sumT \mu' \left([\tilde M]_1^\top  x_*(\ttheta) - \ln(K-1)\right) = \dfrac{T(K+3)}{\kappa_*(\ttheta)} \,.
\end{equation*}
The second term already appears in Equation~\eqref{eq:proof lower bound eq 2} and is therefore bounded by
\begin{equation*}
    \sumT \int_0^1 \mu'\left( [\tilde M]_1^\top (v x_*(\ttheta) + (1-v) x_t) - \ln(K-1) \right) \mathrm{d}v \cdot \left( [\tilde M]_1^\top (x_*(\ttheta) - x_t) \right) \le \dfrac{\sqrt{K+3}}{R} \Reg_T(\ttheta) \,.
\end{equation*}

\medskip \noindent
\textbf{Step 6: Putting Everything Together.} \quad 
We are now ready to carry out the final step of the proof.
We apply Steps~4 and~5 and substitute them in Equation~\eqref{eq:inside proof lower bound 6}
\small
\begin{equation*}
    \sum_{i=2}^d KL \left( \pr_{\ttheta}, \pr_{\flipi(\ttheta)} \right) \le 8 \exp(2\epsilon) \epsilon^2 \left[ 6 \tfrac{\sqrt{K+3}}{R} \Reg_T(\ttheta) + \tfrac{K+3}{4}(d-1) \epsilon^2 \left( \dfrac{T(K+3)}{\kappa_*(\ttheta)} + \tfrac{\sqrt{K+3}}{R} \Reg_T(\ttheta) \right) \right] \,.
\end{equation*}
\normalsize
Using our assumption on $\tilde \cM$, we can now upper-bound $\Reg_T(\ttheta)$ by $Rd\sqrt{KT/\kappa_*(\ttheta)}$. We obtain
\small
\begin{equation*}
    \sum_{i=2}^d KL \left( \pr_{\ttheta}, \pr_{\flipi(\ttheta)} \right) \le 8 \exp(2\epsilon) \epsilon^2 \left[ 6 (K+3)d \sqrt{\dfrac{T}{\kappa_*(\ttheta)}} + \tfrac{K+3}{4}(d-1) \epsilon^2 \left( \dfrac{T(K+3)}{\kappa_*(\ttheta)} + (K+3)d \sqrt{\dfrac{T}{\kappa_*(\ttheta)}} \right) \right] \,.
\end{equation*}
\normalsize
Thus by taking the average over $\ttheta \in \Xi$ we have:
\small
\begin{equation*}
    \dfrac{1}{|\Xi |} \sum_{\ttheta\in\Xi} \sum_{i=2}^d KL \left( \pr_{\ttheta}, \pr_{\flipi(\ttheta)} \right) \le 8 \exp(2\epsilon) \epsilon^2 \left[ 6 (K+3)d \sqrt{\dfrac{T}{\kappa_*(\ttheta)}} + \tfrac{K+3}{4}(d-1) \epsilon^2 \left( \dfrac{T(K+3)}{\kappa_*(\ttheta)} + (K+3)d \sqrt{\dfrac{T}{\kappa_*(\ttheta)}} \right) \right] \,.
\end{equation*}
\normalsize
Hence by substituting in Equation~\eqref{eq:inside proof lower bound 8} we get
\small
\begin{multline*}
    \dfrac{2}{|\Xi |} \sum_{\theta\in\Xi} \sum_{i=2}^d  \pr_{\ttheta}[A_i(M)] \\
    \ge \tfrac{d}{2} - \sqrt{\tfrac{d-1}{2}} \sqrt{ 8\exp(2\epsilon) \epsilon^2 \left[ 6 (K+3)d \sqrt{\dfrac{T}{\kappa_*(\ttheta)}} + \tfrac{K+3}{4}(d-1) \epsilon^2 \left( \dfrac{T(K+3)}{\kappa_*(\ttheta)} + (K+3)d \sqrt{\dfrac{T}{\kappa_*(\ttheta)}} \right) \right] } \,.
\end{multline*}
\normalsize
If this is true for the average over \(\Xi\), then there exists at least one \( \theta_* = \Pi M_* \in \Xi \) such that 
\small
\begin{multline*}
    \sum_{i=2}^d  \pr_{\theta_*}[A_i(M_*)] \\
\end{multline*}
\normalsize
We substitute in Equation~\eqref{eq:inside proof lower bound 7} and get 
\small
\begin{align*}
    &\E_{ \theta_*} [\Reg_T(\theta_*)] \\
    &\ge \dfrac{R(K+3)^{3/2} T \epsilon^2}{64 \kappa_*( \theta_*) } \left[ \tfrac{d}{2} - \sqrt{\tfrac{d-1}{2}} \left[ 8 \exp(2\epsilon) \epsilon^2 \left[ 6 (K+3)d \sqrt{\tfrac{T}{\kappa_*(\theta_*)}} + \tfrac{K+3}{4}(d-1) \epsilon^2 \left( \tfrac{T(K+3)}{\kappa_*(\theta_*)} + (K+3)d \sqrt{\tfrac{T}{\kappa_*(\theta_*)}} \right) \right] \right]^{1/2} \right] \\
    &\ge \dfrac{R(K+3)^{3/2} T \epsilon^2}{64 \kappa_*( \theta_*) } \left[ \tfrac{d}{2} - \tfrac{d}{2} \sqrt{8 \exp(2\epsilon) } \left[ 6 \epsilon^2 (K+3) \sqrt{\tfrac{T}{\kappa_*(\theta_*)}} + \epsilon^4 \tfrac{(K+3)^2}{4} \tfrac{T}{\kappa_*(\theta_*)} + \epsilon^4 \tfrac{(K+3)^2}{4} d \sqrt{\tfrac{T}{\kappa_*(\theta_*)}} \right]^{1/2} \right] \,.
\end{align*}
\normalsize
We choose \( \epsilon^2 = c (K+3)^{-1} \sqrt{\kappa_*( \theta_*)/T} \) with \(c=0.01\) and get
\small
\begin{equation*}
    \E_{ \theta_*} [\Reg_T(\theta_*)] \ge \dfrac{cRd\sqrt{(K+3)T}}{64  \sqrt{\kappa_*( \theta_*)}} \left[ 1 - \sqrt{8 \exp(2\sqrt{c} (K+3)^{-1/2} [\kappa_*( \theta_*)/T]^{1/4} ) } \left[ 6c + \dfrac{1}{4} c^2 + \dfrac{1}{4} c^2 d \sqrt{\dfrac{T}{\kappa_*(\theta_*)}} \right]^{1/2} \right] \,.
\end{equation*}
\normalsize
When \(T\ge d^2 \kappa_*( \theta_*)\) we have that 
\begin{equation*}
    \E_{ \theta_*} [\Reg_T(\theta_*)] \ge \dfrac{cRd\sqrt{(K+3)T}}{64  \sqrt{\kappa_*( \theta_*)}} \left[ 1 - \sqrt{ 8 \exp(\sqrt{c}) } \left[ 6c + \dfrac{1}{4} c^2 + \dfrac{1}{4} c^2 \right]^{1/2} \right] \ge \dfrac{Rd\sqrt{(K+3)T}}{25000  \sqrt{\kappa_*( \theta_*)}} \,.
\end{equation*}

\end{proof}

\section{Removing the Exploration}\label{appendix:removing the exploration}

In this section we introduce a variant of our algorithm with an adaptive exploration and prove its regret bound.

\begin{algorithm}[!ht]
\caption{Using an adaptive exploration}
\label{algo:adaptive exploration}
\KwIn{regularisation parameters \(\lambda^w, \lambda\), learning rate $\eta^w$}
{\bfseries{Init:} \( H_0^w = \lambda^w I_{Kd}, H_1 = \lambda I_{Kd} \)} \\
\For{each time step \(t\) in \(1 \dots T\)}{
    Get action set $\cX_t \subseteq \cX$ \\
    \uIf{$\max_{x\in\cX_t} \lVert I_K \otimes x \rVert^2_{(H_{t-1}^w)^{-1}} \ge 1/\tau_t^2 $}{
        Play $ x_t = \argmax_{x\in\cX_t} \lVert I_K \otimes x \rVert^2_{(H_{t-1}^w)^{-1}} $ \\
        Observe $y_t \sim \mu(\theta_* x_t)$ \\
        Get reward $\rho_{y_t}$ \\
        $ \tilde H_{t}^w \gets H_{t-1} + \tfrac{\eta^w}{\kappa} I_K \otimes x_t x_t^\top $ \\
        $\theta_{t+1}^w \gets \argmin_{\theta \in \R^{K\times d}} \langle \nabla \mu(\theta_t^w) ,  \theta \rangle + \tfrac{1}{2\eta^w} \lVert \theta_t^w - \theta \rVert^2_{\tilde H_t^w} $  \\
        $H_t^w \gets H_{t-1}^w + \tfrac{1}{\kappa} I_K \otimes x_t x_t^\top $ \\
        $ \cW_{t+1}(\delta) \gets \{ \theta \in \R^{K\times d} : \lVert \theta - \theta_{t+1}^w \rVert_{H_t^w} \le \beta_{t+1}(\delta) \} $
    }
    \Else{
        Play $ x_t = \argmax_{x\in\cX_t} \tilde r_t(x) $ with \(\tilde r_t(x) \) defined in Eq.~\eqref{eq:optimistic reward} \\
        $\tilde H_{t+1} \gets H_t + \nabla\mu(\theta_t x_t) \otimes x_t x_t^\top $ \\
        $\theta_{t+1} \gets \argmin_{\theta \in \cW_t(\delta)} \langle \nabla\ell_{t+1}(\theta_t), \theta \rangle + \lVert \theta - \theta_t \rVert^2_{\tilde H_{t+1}} $ \\
        $H_{t+1} \gets H_t + \nabla\mu(\theta_{t+1} x_t) \otimes x_t x_t^\top $ \\
        $ \bar H_{t+1} \gets H_{t+1} + 1_K1_K^\top \otimes x_t x_t^\top $ \\
        $ \cW_{t+1}(\delta) \gets \cW_t(\delta) $ \\
        $ \theta_{t+1}^w \gets \theta_t^w $ \\
        $ H_{t}^w \gets H_{t-1}^w $
    }
}
\end{algorithm}

\subsection{Proof of Theorem~\ref{thm:adaptive regret bound}}\label{appendix:proof adaptive regret bound}

In this section we prove Theorem~\ref{thm:adaptive regret bound}, the regret upper-bound of Algorithm~\ref{algo:adaptive exploration}. We start by studying the exploration part of the algorithm.

\subsubsection{Analysis of the Adaptive Exploration}

We start by showing that at each iteration of the algorithm, the set $\cW_t(\delta)$ is a confidence set.
\begin{restatable}{lmm}{}\label{lemma:adaptive exploration confidence set}
Let $\delta \in (0,1], \eta^w = (1+\sqrt{6}S)/2$ and $ \lambda^w = 144 \eta^w Kd $.
Let us define $\beta_t(\delta) = 4 S \sqrt{Kd\log(t/\delta)} + 2 S \sqrt{\lambda^w}$. Then we have with probability
$1-\delta$, for all $t\ge 1$, 
\begin{equation*}
    \theta_* \in \cW_{t}(\delta) \,.
\end{equation*}
\end{restatable}

\begin{proof}
Let $t\in \llbracket T \rrbracket$.
For all $x\in \cX_t$, using Cauchy-Schwarz inequality we have
\begin{equation*}
    \max_{\theta\in\cW} \lVert ( \theta - \theta_*) x \rVert_2 \le \max_{\theta\in\cW} \lVert \theta - \theta_* \rVert_2 \lVert x \rVert_2 \le 2 S \,.
\end{equation*}
We apply \citep[Theorem~4.2]{lee2025improved} with $\alpha=2S$ and get $ \beta_t(\delta) = 4 S \sqrt{Kd\log(t/\delta)} + 2 S \sqrt{\lambda^w} $.
\end{proof}

We now upper-bound the number of exploration steps to show it is negligible in the regret.

\begin{restatable}{lmm}{}\label{lemma:bound number exploration steps}
    Let $T^w$ the set of exploration steps. We have
    \begin{equation*}
        |T^w| \le 2 \tau^2_t \kappa Kd \log\left( 1+ \dfrac{T}{Kd\lambda^w} \right) \,.
    \end{equation*}
\end{restatable}

\begin{proof}
We start with Trace-Determinant argument to upper-bound the following sum:
\begin{align*}
    \sumTw &\max_{x\in \cX_t} \lVert I_K \otimes x \rVert^2_{(H_{t-1}^w)^{-1}} \\
    &= \sumTw \lVert I_K \otimes x_t \rVert^2_{(H_{t-1}^w)^{-1}} \\
    &= \kappa \sumTw \dfrac{1}{\kappa} \lVert I_K \otimes x_t \rVert^2_{(H_{t-1}^w)^{-1}} \\
    &= \kappa \sumTw \lVert \kappa^{-1/2} I_K \otimes x_t \rVert^2_{(H_{t-1}^w)^{-1}} \\
    &\le 2\kappa Kd \log\left(1 + \dfrac{T}{Kd \lambda^w} \right) \,. &&\text{\citep[Lemma~10]{abbasi2011improved}}
\end{align*}
We now lower-bound this sum using the exploration rule:
\begin{equation*}
    \sumTw \max_{x\in \cX_t} \lVert I_K \otimes x \rVert^2_{(H_{t-1}^w)^{-1}} \ge \sumTw \dfrac{1}{\tau_t^2} = |T^w| \dfrac{1}{\tau_t^2} \,.
\end{equation*}
Therefore we have
\begin{equation*}
    |T^w| \le 2 \tau^2_t \kappa Kd \log\left( 1+ \dfrac{T}{Kd\lambda^w} \right) \,.
\end{equation*}
\end{proof}

Finally we bound the diameters of the confidence sets $\cW_t(\delta)$. It will allow us to leverage the self-concordance property for a constant cost.

\begin{restatable}{lmm}{}\label{lemma:adaptive constant diameter}
    Let us define $ \tau_t = 2\sqrt{6}\beta_t(\delta) $. Let $\delta \in (0,1]$, with probability $1-\delta$, for all $t\ge 1$ we have
    \begin{equation*}
        \max_{x\in\cX_t} \max_{\theta_1, \theta_2\in\cW_t(\delta)} \lVert (\theta_1-\theta_2) x \rVert_2 \le \dfrac{1}{\sqrt{6}} \,.
    \end{equation*}
\end{restatable}

\begin{proof}
Let $t\in\llbracket T \rrbracket$. For all $x\in\cX_t$, using the Cauchy-Schwarz inequality and the Triangle inequality we have
\begin{align*}
    \max_{\theta_1, \theta_2\in\cW_t(\delta)} &\lVert  (\theta_1-\theta_2) x \rVert_2 \\
    &= \max_{\theta_1, \theta_2\in\cW_t(\delta)} \lVert (I_K \otimes x) (\theta_1-\theta_2) \rVert_2 \\
    &\le \lVert I_K \otimes x \rVert_{(H_{t-1}^w)^{-1}} \max_{\theta_1, \theta_2\in\cW_t(\delta)} \lVert \theta_1 - \theta_2 \rVert_{H_{t-1}^w} &&\text{(CS)} \\
    &\le \max_{x\in\cX_t}  \lVert I_K \otimes x \rVert_{(H_{t-1}^w)^{-1}} \left(\max_{\theta_1, \theta_2\in\cW_t(\delta)} \lVert \theta_1 - \theta^w_t \rVert_{H_{t-1}^w} + \lVert \theta^w_t - \theta_2 \rVert_{H_{t-1}^w} \right) \\
    &\le \dfrac{1}{\tau_t} 2 \beta_t(\delta) &&\text{(Lemma~\ref{lemma:adaptive exploration confidence set})} \\
    &= \dfrac{1}{\sqrt{6}} 
\end{align*}
where the last equality is by definition of $\tau_t$.
\end{proof}

\subsubsection{Regret Upper-bound}

We now focus on the learning part of the algorithm. We start by showing that at each iteration of the algorithm $\sigma_t(\delta)$ defines a confidence set.

\begin{restatable}{lmm}{}\label{lemma:adaptive learning confidence set}
    Let $\delta \in (0,1], \eta = 1$ and $ \lambda = 144Kd $. Let us define $\beta_t(\delta) = 4 S \sqrt{Kd\log(t/\delta)} + 2 S \sqrt{\lambda^w}, \sigma_t(\delta) = 2 \sqrt{Kd \log(t/\delta)} + 24 S \sqrt{Kd} $ and $\tau_t = 2 \sqrt{6} \beta_t(\delta)$. Then we have with probability $1-2\delta$, for all $t\ge 1$, 
\begin{equation*}
    \lVert \theta_* - \theta_{t+1} \rVert_{\bar H_{t+1}} \le \sigma_{t}(\delta) \,.
\end{equation*}
\end{restatable}

\begin{proof}
First, by Lemma~\ref{lemma:constant diameter}, we can apply \citep[Theorem~4.2]{lee2025improved} with $\alpha=1/\sqrt{6}$ to obtain
\begin{equation*}
    \lVert \theta_* - \theta_{t+1} \rVert_{H_{t+1}} \le \sigma_{t}(\delta) \,.
\end{equation*}
Then, we decompose \(\R^K\) as \( \R^K = 1_K \oplus \cH \) where \(\cH\) is the hyperplane supported by \(1_K\). 
Recall that \(\theta_\ast, \theta_{t+1} \in \Pi \R^{K\times d} \), for all \(x\in\cX\), 
by definition of \(\Pi\), \(\theta_{t+1} x\) and \( \theta_\ast x \) are in \(\cH\). 
Therefore \(  \sum_{s=1}^t \lVert (\theta_{t+1} - \theta_\ast)x_s \rVert_{1_K 1_K^\top} = 0 \). 
And we conclude that with probability \(1-2\delta\)
\begin{equation*}
    \lVert \theta_{t+1} - \theta_\ast \rVert_{\bar H_{t+1}} = \lVert \theta_{t+1} - \theta_\ast \rVert_{H_{t+1}} \le \sigma_t(\delta) \,.
\end{equation*}

\end{proof}

We can now recall and prove our regret upper-bound for Algorithm~\ref{algo:adaptive exploration}.

\ThmAdaptiveRegretBound*

\begin{proof}
\nopagebreak

\medskip\noindent 
\textbf{Step 1: Tackling the exploration part.} \quad 
We separate the regret from the exploration and the regret from the learning part:
\begin{align*}
    \Reg_T &:= \sumT \rho^\top \mu(\theta_* x_{*,t}) - \rho^\top \mu(\theta_* x_t) \\
    &= \sumTw \rho^\top \mu(\theta_* x_{*,t}) - \rho^\top \mu(\theta_* x_t) + \sum_{t\notin T^w} \rho^\top \mu(\theta_* x_{*,t}) - \rho^\top \mu(\theta_* x_t) \\
    &\le R |T^w| + \sum_{t\notin T^w} \rho^\top \mu(\theta_* x_{*,t}) - \rho^\top \mu(\theta_* x_t) \\
    &\le 2 R \tau^2_t \kappa Kd \log\left( 1+ \dfrac{T}{Kd\lambda^w} \right) + \sum_{t\notin T^w} \rho^\top \mu(\theta_* x_{*,t}) - \rho^\top \mu(\theta_* x_t) 
\end{align*}
where the last inequality is due to Lemma~\ref{lemma:bound number exploration steps}.
Let us now focus on the learning phase of the algorithm.

\medskip \noindent
\textbf{Step 2: Using optimism.} \quad 
Using the definition of the optimistic reward we can bound the regret twice
\begin{align*}
    \Reg_T(\mathrm{Learning}) &:= \sumnTw \rho^\top (\mu(\theta_*x_{*,t}) - \mu(\theta_* x_t)) \\
    &\le \sumnTw \rho^\top \mu(\theta'_t x_{*,t}) + \epsilon_{1,t}(x_{*,t}) + \epsilon_{2,t}(x_{*,t}) - \rho^\top \mu(\theta_* x_t) && \text{(Prop.~\ref{prop:optimistic reward})} \\
    &\le \sumnTw \rho^\top \mu(\theta'_t x_t) + \epsilon_{1,t}(x_t) + \epsilon_{2,t}(x_t) - \rho^\top \mu(\theta_* x_t) &&\text{(Def. of \(x_t\))} \\
    &\le 2 \sumnTw \epsilon_{1,t}(x_t) + 2 \sumnTw \epsilon_{2,t}(x_t) &&\text{(Prop.~\ref{prop:optimistic reward})} \,.
\end{align*}

\medskip \noindent
\textbf{Step 3: Concluding.}
We may now follow our proof of Theorem~\ref{thm:regret bound} to obtain with probability $ 1-2\delta $
\begin{align*}
    \Reg_T(\text{Learning}) &\le \dfrac{24 \sigma_T(\delta) \kappa K d R }{\lambda} \log\left( 1 + \dfrac{T}{K \lambda} \right) 
    + 12 \kappa \sigma_T(\delta)^2 Kd \log\left( 1+ \dfrac{T}{Kd\lambda\kappa} \right) \\
    &\qquad + 4\sqrt{2} e \sigma_T(\delta) \sqrt{d \log\left( 1 + \dfrac{T}{\lambda d}  \right)} R \sqrt{\sumnTw \dfrac{1}{\kappa_{\ast, t}} } \\
    &\qquad + 16 e^2 \sigma_T(\delta)^2 d \log\left( 1 + \dfrac{T}{\lambda d}  \right) (4R+2\nu)  \\
    &\lesssim \cste \kappa (R+\nu) K^2 d^2 + \cste  d R \sqrt{ \sumnTw \dfrac{K}{\kappa_{\ast, t}} }  \,.
\end{align*}
\end{proof}

\section{Auxiliary Results}\label{appendix:auxiliary results}
\begin{restatable}{lmm}{LmmAux1}[\citet[Section~4.2.3]{boyd2004convex}]\label{lemma:minimiser to gradient positive}
    Let \(f:\R^d\to\R\) be a convex and differentiable function and \(\mathcal{C} \subseteq \R^d\) a convex set. Further, denote:
    \begin{equation*}
        x_0 := \argmin_{x_0\in\mathcal{C}} f(x) \,.
    \end{equation*}
    Then for any \(y\in\mathcal{C}\):
    \begin{equation*}
        \nabla f(x_0)^\top  (y-x_0) \ge 0 \,.
    \end{equation*}
\end{restatable}

\begin{restatable}{lmm}{LmmAux2}[Modified Freedman's Inequality,  \citet[Lemma~3]{lee2024improved}]\label{lemma:modified Freedman inequality}
    Let \( X_1, \dots, X_t \) be a martingale difference sequence satisfying \(\max_s |X_s| \le D \) a.s., and let \(\F_s \) be the \(\sigma\)-field generated by \( (X_1, \dots, X_s) \). Then for any \(\delta\in(0, 1] \) and any \( \eta \in [0, 1/D]  \) the following holds with probability \(1-\delta\)
    \begin{equation*}
        \sum_{s=1}^t X_s \le (e-2) \eta \sum_{s=1}^t \E[X_s^2 | \F_{s-1}] + \dfrac{1}{\eta} \log \delta^{-1} \qquad \forall t \ge 1 \,.
    \end{equation*}
\end{restatable}

\begin{restatable}{lmm}{LmmDetTraceInequality}[Determinant-Trace Inequality,  \citet[Lemma~10]{abbasi2011improved}]\label{lemma:det_trace inequality}
    Let \( \{x_s\}_{s=1}^{\infty} \) be a sequence in \(\R^d\) such that \(\lVert x_s \rVert_2 \le X \) for all \(s\ge1\), and let \(\lambda\ge 0\). For \( t\ge1 \) define \( V_t := \sum_{s=1}^t x_s x_s^\top  + \lambda I_d \). The following inequality holds:
    \begin{equation*}
        \det(V_t) \le (\lambda + t X^2 / d )^d \,.
    \end{equation*}
\end{restatable}

\end{document}